\documentclass{article}
\usepackage[utf8]{inputenc}
\usepackage[final,nonatbib]{neurips_2021}
\input{style.sty}

\title{Beyond Tikhonov: Faster Learning with Self-Concordant Losses via Iterative Regularization}

\author{%
        Gaspard Beugnot \\
        Inria\thanks{Inria, \'Ecole normale sup\'erieure, CNRS, PSL Research University, 75005 Paris, France} \\
        \hspace*{-0.2cm}\texttt{gaspard.beugnot@inria.fr} 
        \And
        Julien Mairal \\
        Inria\thanks{Inria, Univ. Grenoble Alpes, CNRS, Grenoble INP, LJK, 38000 Grenoble, France} \\
        \texttt{julien.mairal@inria.fr} 
        \And
        Alessandro Rudi \\
        Inria\footnotemark[1] \\
        \hspace*{-0.1cm}\texttt{alessandro.rudi@inria.fr}
        }

\begin{document}

{\hypersetup{hidelinks} 
\maketitle
}
\begin{abstract}
The theory of spectral filtering is a remarkable tool to understand the statistical properties of learning with kernels. For least squares, it allows to derive various regularization schemes that yield faster convergence rates of the excess risk than with Tikhonov regularization. This is typically achieved by leveraging classical assumptions called source and capacity conditions, which characterize the difficulty of the learning task.
In order to understand estimators derived from other loss functions, Marteau-Ferey et al. \cite{regularized-erm-gsc} have extended the theory of Tikhonov regularization to generalized self concordant loss functions (GSC), which contain, {\it e.g.}, the logistic loss.
In this paper, we go a step further and show that fast and optimal rates can be achieved for GSC by using the iterated Tikhonov regularization scheme, which is intrinsically related to the proximal point method in optimization, and overcomes the limitation of the classical Tikhonov regularization.
\end{abstract}

\section{Introduction}

We consider the problem of supervised learning where we want to find a prediction function~$\theta$ mapping an input point $x$ living in a set~$\XX$ to a  label $y$ in $\YY$.  In this paper, we assume that $\theta$ lives in a separable Hilbert space~$\HH$ and is learned from a set of observations $(x_i,y_i)_{i=1,\ldots,n}$ that are i.i.d. samples drawn from an unknown probability distribution $\rho$ on $\XY$. The goal is to find $\theta$ that minimizes the expected risk~$L$, which is defined below along with the empirical risk $\hat{L}$:
\begin{equation}\label{eq:risk_and_empirical_risk}
    \lossL{}{}{\theta} = \int_{\XX \times \YY} \ell(y, \theta(x)) \dd \rho(x, y), \qquad \lossLEmp{}{}{\theta} = \frac{1}{n} \sum_{i=1}^n \ell(y_i, \theta(x_i)),
\end{equation}
where $\ell$ is a suitable loss function comparing true labels with predictions.
This paper aims for upper bounds on the excess risk for a specific estimator $\widehat{\theta}$. That is, we assume that the minimum of the expected risk is attained for some $\optimum$ in $\HH$, and we want to derive \textit{probabilistic upper bounds on the excess risk}: 
\begin{equation}\label{eq:upper_bound_we_seek}
    \PP\csb{\lossL{}{}{\estEmp{}{}} - \lossL{}{}{\optimum} > C_1 n^{-\gamma} \, \log\tfrac{2}{\delta}} \leq \delta,
\end{equation}
given some value $\delta$ in $(0, 1)$, where $C_1$ is a positive constant, and $\estEmp{}{}$ is an estimator built from the $n$ observations. The quantity $O(n^{-\gamma})$ denotes the rate of convergence of the estimator $\hat{\theta}$. 
A classical ``slow'' rate with $\gamma = 1/2$ is typically achieved by many estimators and is in fact optimal if only mild assumptions are
made about the data distribution~$\rho$. Even though optimal, this rate is nevertheless a worst case and faster rates with $\gamma > 1/2$ can
be achieved both in theory and in practice, by making additional assumptions about the difficulty of the learning task. Originally introduced in the literature of inverse problems, the so-called \emph{source} and \emph{capacity} conditions have been shown to be appropriate 
for this purpose, leading to statistical analysis with fast rates of convergence~\cite{regularized-erm-gsc,pub.1045202542,optimal_rate_rls}.
The optimality of results of the form~\eqref{eq:upper_bound_we_seek} is characterized
by comparing them with lower bounds that are available for various sets of data distributions~$\rho$~\cite{optimal_rate_rls}. Matching upper bounds with lower bounds ensures that the estimator $\fontsize{9}{0}\widehat{\theta}$ is \textit{optimal}, in the sense that no information is lost in the process of exploiting the data samples to compute $\widehat{\theta}$, for the given set of distributions.

In this search for optimal estimators, most of the attention has been devoted to minimizers of some function of the empirical risk $\lossLEmp{}{}{}$, which is defined in (\cref{eq:risk_and_empirical_risk}). Then, the key challenge is to \textit{regularize} $\lossLEmp{}{}{}$ in order to achieve better generalization properties. The most widely used scheme is probably Tikhonov regularization; other examples when $\HH$ is a RKHS include truncated regression \cite{kernel_algs}, or early stopping in gradient descent algorithms \cite{Yao2007,averyanov2020early}. When the loss $\ell$ is set to least squares, it can be shown that minimizing the excess risk amounts to solving an ill-posed inverse problem \cite{JMLR:v6:devito05a}, which led to the remarkable theory of \textit{spectral filtering}. A large class of regularization schemes can indeed be seen as a filtering process applied to the training labels $y_i$ after regularizing the spectrum of the kernel matrix \cite{pub.1045202542,bauer2007regularization}. Interestingly, this theory has highlighted the fact that not all regularization schemes are equal: some of them obtain fast learning rates in \eqref{eq:upper_bound_we_seek} on ``easy'' problem (a thorough definition is given in \cref{sec:settings}) while others cannot leverage this additional regularity to improve the learning rate.

Such a general analysis for least squares is made possible by the fact that a closed-form expression of the estimator is available. When considering different loss function $\ell$, the estimator $\widehat{\theta}$ is unfortunately only implicitly available as the solution of an optimization problem involving $\lossLEmp{}{}{}$. 
A step to extend least squares results to more general loss functions has been achieved by Marteau-Ferey et al. \cite{regularized-erm-gsc}, who provide bounds on the form \eqref{eq:upper_bound_we_seek} for Tikhonov estimator on generalized self concordant (GSC) functions. GSC functions are three-times-differentiable functions whose third derivative is bounded by the second-derivative. In practice, they were introduced to conduct a general analysis of the Newton method in optimization \cite{boyd_vandenberghe_2004,nesterov1994interior}, and adapted in \cite{10.1214/09-EJS521} to encompass a larger class of loss function. It includes notably the logistic regression loss, which is widely used for classification. 

While Tikhonov yields fast rates of convergence in several data regimes, it is known to be unable to adapt to the whole range of learning task difficulties. More precisely, it suffers from a ``saturation'' effect \cite{pub.1045202542}, meaning that when the learning task becomes simpler, the learning rate stops improving and is suboptimal. Our paper addresses this limitation for GSC functions by considering instead the iterated Tikhonov regularization (IT) scheme. In the context of least squares, this approach consists of successively fitting the residuals. For more general loss functions, it is equivalent to performing a few steps of the proximal point method in optimization~\cite{rockafellarppa}. 
Our main result is a probabilistic upper bound on the excess risk, which is optimal given usual source and capacity conditions assumptions on the learning task, thus addressing the limitations of the classical Tikhonov regularization. 

\section{Background and Preliminaries}
\label{sec:settings}

\subsection{Definitions: Estimator and Loss Function}
Let $\XX$ be a Borel input space, $\YY$ be a vector-valued output spaces, and
$\rho$ a probability distribution on $\XX \times \YY$. We consider $\HH$ to be a separable Hilbert space of functions from $\XX$ to $\YY$.
Given a loss function $\ell: \YY \times \YY \to \RR$, we aim at minimizing the expected loss, while we only have access to the empirical loss -- both are defined in \cref{eq:risk_and_empirical_risk}.
Our work provides an upper bound on the excess risk of the iterated Tikhonov estimator. For the basic case of least squares with $\YY=\RR$, it is usually defined as a procedure that refits the residuals, see, e.g., \S 5.4 in \cite{pub.1045202542}. Starting with {$ \estEmp{\lambda}{0}=0$}, it consists of the sequence
\begin{equation}
    \estEmp{\lambda}{t} = \estEmp{\lambda}{t-1} + \argmin_{\theta \in \HH} \left\{ \frac{1}{n} \sum_{i=1}^n \frac{1}{2} \left(y_i - \estEmp{\lambda}{t-1}(x_i) - \theta(x_i)\right)^2 + \frac{\lambda}{2} \norm{\theta}^2 \right\}.
\end{equation}
To extend this regularization to other loss function, we make the change of variable $\theta' = \estEmp{\lambda}{t-1} + \theta$ in the equation above, which yields the proximal point algorithm~\cite{rockafellarppa}.
\begin{definition}[Iterated Tikhonov estimator a.k.a. proximal point algorithm]\label{def:it_estimator_main}
    We define the iterated Tikhonov estimator with the following sequence. Given $\lambda > 0$ and $\estEmp{\lambda}{0}=0$,
    \begin{equation}\label{eq:it_estimator_definition}
            \estEmp{\lambda}{t+1} = \prox{\lossLEmp{}{}{}/\lambda} (\estEmp{\lambda}{t})  \eqdef \argmin_{\theta \in \HH} \left\{\lossLEmp{}{}{\theta} + \frac{\lambda}{2} \left\|\theta - \estEmp{\lambda}{t}\right\|^{2}\right\},
    \end{equation}
    where $\prox{\lossLEmp{}{}{}/\lambda}$ denotes the proximal operator of the empirical risk $\lossLEmp{}{}{}$ rescaled by $1/\lambda$.
\end{definition}
\begin{remark}
    In practice, the proximal operator is only computed approximately by using an optimization algorithm. Nevertheless, the benefits in terms of statistical accuracy of the iterated Tikhonov scheme are robust to inexact solutions, as long as the accuracy for solving the sub-problems~(\cref{eq:it_estimator_definition}) is high enough. We discuss this point in \cref{subsec:optim}.
\end{remark}
\begin{remark}
    It is easy to show that the sequence of the proximal point algorithm always converges to a minimizer of the unregularized empirical risk, which is of course not what we are interested in. Instead, we consider and analyze the procedure with a fixed small number of steps $t$ and show later that optimal learning rates can be obtained by choosing an appropriate parameter $\lambda$.
\end{remark}
\begin{remark}
When the loss is a function of a residual $y-\theta(x)$---assuming $\YY$ to be a vector space---as in the least square case, we recover the classical definition consisting of refitting the residual, and with $t=1$, we recover Tikhonov. 
\end{remark}

Interestingly, our definition makes the estimator compatible with other loss functions, such as the logistic loss.  
More precisely, the main assumption we make on the loss is to be \textit{generalized self concordant}. We follow the definition of \cite{regularized-erm-gsc}, which is a special case of 2-self concordance introduced in \cite{sun2018generalized}:

\begin{definition}[Generalized self-concordance]\label{def:gsc}
    For any $z = x, y \in \XX \times \YY$, the function $\ell_z: \HH \to \RR$ defined as $\ell_z(\theta)=\ell(y,\theta(x))$ is convex and three times differentiable. Besides, there exists a set $\phi(z) \subseteq \HH$ s.t:
    \begin{equation}\label{eq:def_gsc}
     \forall \theta, h,k \in \HH, \quad \absv{\nabla^3\ell_z(\theta) \csb{h, k, k}} \leq \sup_{g \in \phi(z)} \absv{k \cdot g} \nabla^2 \ell_z (\theta) \csb{k, k}.
    \end{equation}
\end{definition}
The brackets indicate that the vectors $h, k$ and $k$ are applied to the 3-dimensional tensor $\nabla^3\ell_z(\theta)$.
The definition seems technical at first sight, but
intuitively, this assumption allows to upper bound the deviation between the objective function and its local quadratic approximation. This enables a simple analysis of the Newton method for optimization, making it easy to quantify the basin of quadratic convergence \cite{NEURIPS2019_60495b4e}. On top of this, it has the benefit of encompassing a large class of loss functions, such as the logistic loss: see Example 1 in \cite{regularized-erm-gsc} for values of $\phi(z)$ with usual losses. We provide some intuition on GSC loss functions in \cref{rem:intuition_gsc} in \cref{sec:definitions_sec_inexact_solvers}.

In order to ensure the existence of the loss and its derivatives everywhere, we also need the following technical assumptions also introduced in~\cite{regularized-erm-gsc}, which are reasonable in practice. This ensures that both $\lossL{}{}{}$ and $\lossLEmp{}{}{}$ are generalized self concordant too.
\begin{assumption}[Technical assumptions]\label{hyp:technical_assumptions_main}
    There exists $R$ s.t $\sup_{g \in \phi(z)} \norm{g} \leq R$ almost surely for $z$ drawn from the distribution $\rho$ and $\absv{\ell_z(0)}, \norm{\nabla \ell_z(0)}, \Tr \nabla^2 \ell_z(0)$ are almost surely bounded.
\end{assumption}
The following assumption is usual in excess risk analysis \cite{regularized-erm-gsc,NIPS2015_03e0704b}. In our proof strategy, all the quantities are vectors and operators in $\HH$, which makes the analysis simpler. Weakening this assumption (\textit{e.g.} assuming that $\optimum \in \mathcal{L}_2(\XX)$) would require finding an equivalent of the covariance operator for GSC loss function, which constitute an interesting future direction.
\begin{assumption}[Existence of a minimizer]\label{hyp:existence_minimizer}
    There exists $\optimum$ in $\HH$ s.t $\lossL{}{}{\optimum} = \inf_{\theta \in \HH} \lossL{}{}{\theta}$.
\end{assumption}
Finally, following \cite{regularized-erm-gsc} we also
define the \textit{expected Hessian} and the \textit{regularized expected Hessian} as
\begin{equation*}
    \forall \theta \in \HH, \, \lambda > 0, \quad \Hess{}{}{\theta} = \EE_{z \sim \rho} \csb{\nabla^2 \ell_z(\theta)}, \quad \Hess{\lambda}{}{\theta} = \Hess{}{}{\theta} + \lambda\II,
\end{equation*}

and we introduce the degrees of freedom, also known as the effective dimension of the problem: 
\begin{definition}[Degrees of freedom]\label{def:degree_freedom}
    The degrees of freedom is defined as:
    \begin{equation*}
        \forall \lambda, \quad \df_\lambda = \EE_{z \sim \rho} \csb{\norm{\nabla \ell_z(\optimum)}_{\Hess{\lambda}{-1}{\optimum}}^2}.
    \end{equation*}
    where we denote by $\norm{\theta}_A = \norm{A^{1/2}\theta}$, with $\theta \in \HH$, the norm induced by a positive definite operator $A$ on $\HH$.
\end{definition}

\begin{remark}\label{remark:ls}
The intuition about this definition is not straightforward. To better understand why this quantity is a key to characterize the amount of regularization in a learning problem, it is useful to consider the specific case of the square loss with kernels. In such a case, $\HH$ is a reproducing kernel Hilbert space (RKHS) and $\theta(x)= \theta^\top \Phi(x)$, where $\Phi: \XX \to \HH$ is the kernel mapping. Then, the Hessian is constant everywhere and equal to the \textit{covariance operator} $T = \EE_{x \sim \rho_x}\csb{\Phi(x) \otimes \Phi(x)}$ where $\rho_x$ is the marginal of $\rho$. Consequently, the degrees of freedom (also known as {\em effective dimension}) is a spectral function of $T$ which may be written as $\df_\lambda = \Tr T T_\lambda^{-1}$. This is the classical quantity which appears on the bias/variance decomposition of the excess risk, with a variance part decaying in $\Tr T T_\lambda^{-1}/n$, see~\cite{optimal_rate_rls}. 
\end{remark}

\subsection{Source and Capacity Conditions}
We now introduce the hypotheses we make on the learning task, which will allow us to derive fast rates of convergence. They measure the difficulty of the problem and are classical in the context of learning with kernels, see \textit{e.g.} \cite{bauer2007regularization,NIPS2017_61b1fb3f,blanchard}. 
It is indeed established that given an algorithm which outputs an estimator $\estEmp{}{}$, one can find a probability measure $\rho$ s.t the learning rate of the estimator is arbitrarily low, a result known as the ``no-free lunch theorem'' \cite{gyorfi2006distribution}. Inspired by the literature of inverse problems, two assumptions were introduced to restrict the space of considered distributions.
\begin{assumption}[Source condition]\label{hyp:source}
    There exists $r > 0$ and $v$ in $\HH$ s. t: $\optimum = \Hess{}{r}{\optimum} v$.
\end{assumption}
$A \mapsto A^r$ is the usual power for positive definite operators. The source condition should be seen as a smoothness assumption on $\optimum$, and for least square, 
we recover the usual definition of the source condition, that is $\optimum = T^r v$, with $T$ the covariance operator we previously defined. Bigger $r$ implies that the optimum can be well approximated by a few eigenvectors. Assuming $r=0$ simplifies to $\optimum \in \HH$.

The second assumption characterizes the ill-posedness of the problem:
\begin{assumption}[Capacity condition]\label{hyp:capacity}
    There exists $\alpha >1$, $\cstSsmall, \cstS > 0$ s.t ~
    $
    \cstSsmall \lambda^{-1/\alpha} \leq \df_\lambda \leq \cstS \lambda^{-1/\alpha}. 
    $
\end{assumption}
Again, for the square loss, it turns to a bound on the eigenvalue decay of the covariance operator. If $\sigma_j, e_j$ is an eigenbasis of $T$, then $\sigma_j = O(j^{-\alpha})$. Said differently, the bigger $\alpha$, the fewer directions are needed to approximate well a sample $x \sim \rho_x$ in expectation, and the easier is the learning task. This is an assumption on the input space $\XX$ and does not imply anything on the labels $\YY$.

\subsection{Previous Results}

Our main result considers iterated Tikhonov \textit{with} GSC loss functions. While iterated Tikhonov has been previously analyzed for squared loss by leveraging the theory of spectral filtering (see below), extensions to other loss functions raise several difficulties, which will be detailed in \cref{sec:results}.

\paragraph{Spectral filters and least squares.}
As we mentioned earlier, the key insight on regularization with the square loss is that a closed-form expression of the estimator is available. By using the same notation as in \cref{remark:ls}, the kernel ridge regression estimator can be for instance written
\begin{equation}\label{eq:representer_theorem_and_spectral_filter}
    \estEmp{\lambda}{} = \sum_{i=1}^n \beta_i \Phi(x_i) ~~~~\text{with}~~~~ \beta = \frac{1}{n}g_\lambda\p{\frac{K}{n}} y, 
\end{equation}
where $K$ is the $n \times n$ kernel matrix, $y=(y_i)_{1\leq i \leq n}$ is the vector of training labels and $g_\lambda(K/n) = (K/n + \lambda I)^{-1}$. 
Note that $g_\lambda$ is a function acting on the spectrum of $K$, which makes it a special case of regularization by \emph{spectral filtering}, which may be analyzed for more general functions~$g_\lambda$. %
In particular, a key quantity for understanding the regularization effect of  a filter $g_\lambda$ is the so-called \emph{qualification}. Following \cite{pub.1045202542,bauer2007regularization}, this quantity is defined below.
\begin{definition}[Qualification of a spectral filter]
    For any $\lambda>0$, define $g_\lambda: \csb{0, 1} \to \RR$ a filter function. Its qualification is the highest $q$ such that
    \begin{equation}
        \forall \nu \leq q, \quad \sup_{\sigma} \absv{1 - \sigma g_\lambda(\sigma)}\sigma^\nu \leq \omega_\nu \lambda^\nu,
    \end{equation}
    with $\omega_\nu$ a constant independant of $\lambda$. 
\end{definition}
Under the source and capacity conditions, it is possible to show that the resulting estimator would enjoy an optimal rate in $n^{-\frac{\alpha(1+2r)}{1 + \alpha(1+2r)}}$ if $r + 1/2\leq q$ (where $r$ comes from the source condition). When $r +1/2 > q$, the rate is instead of order  $n^{-\frac{\alpha(1+2q)}{1 + \alpha(1+2q)}}$, which is suboptimal, see \textit{e.g.} Thm. 3.4 \cite{blanchard} (set the parameter $s$ to $1/2$). This illustrates the \emph{saturation effect} of some regularization schemes. For example, Tikhonov regularization amounts to filtering with $g_\lambda: \sigma \mapsto (\sigma + \lambda)^{-1}$ and has \emph{qualification} $1$, so the parameter $r$ \textit{saturates} at $r=1/2$. Thus, even if $r \gg 1/2$, the excess risk of $\estEmp{\lambda}{}$ will decay in $n^{-\frac{\alpha}{1 + \alpha}}$, which is suboptimal. Designing estimators with high qualification is key to obtaining fast rates that can adapt to both hard and easy learning tasks. 

\paragraph{Iterated Tikhonov with the Square Loss.}
We can compute the spectral filter function $g_\lambda^t$ corresponding to $t$ iterations of IT, which yields 
\begin{equation}\label{eq:it_spectral_function}
    g_\lambda^t : \sigma \mapsto (\sigma + \lambda)^{-1} \sum_{i=0}^{t-1} \p{\frac{\lambda}{\sigma+\lambda}}^{i} = \sigma^{-1}\p{1 - \p{\frac{\lambda}{\sigma + \lambda}}^{t}}.
\end{equation}
Choosing a fixed $t$ and computing the supremum of $\sigma \mapsto \absv{1 - \sigma g_\lambda(\sigma)}\sigma^\nu$, we find that IT estimator has qualification $t$, which is thus better than Tikhonov. IT has been thoroughly studied in the community of inverse problems, dating back to the work of \cite{doi:10.1080/01630567908816031}. It was naturally transferred to learning with kernels thanks to the aforementioned connection with inverse problems. 

The link we make with the proximal point algorithm has never been studied from a statistical perspective, to the best of our knowledge, even though it has attracted a lot of attention in the optimization literature, notably with accelerated algorithms \cite{lin2018catalyst,JMLR:v21:19-073}, or variants of the proximal operator on a class of self-concordant loss functions \cite{DBLP:conf/nips/CarmonJJJLST20}.
More attention was devoted to \textit{boosting}, where the penalty $\lambda$ is fixed but the number of iterations $t$ may go to infinity, necessitating an appropriate stopping rule \cite{lin2019boosted}. Nevertheless, such a work focuses on the least square loss, where the theory of spectral filter can be applied.
Finally, the proximal sequence in \cref{eq:it_estimator_definition} can be cast as a constrained optimization problem related to sequential greedy approximation \cite{1184144}.

\paragraph{Tikhonov and Generalized Self Concordant losses.}
Extending the results obtained with the square loss to more general losses is challenging since there is no closed form available for the resulting estimator, and the 
theory of spectral filtering does not apply.
Nevertheless, the case of Tikhonov regularization for GSC loss functions was treated in \cite{regularized-erm-gsc}. It is shown that the resulting estimator enjoys optimal rate as long as $r \leq 1/2$, meaning that the saturation of Tikhonov regularization is recovered in those settings. We will extend these results to the IT regularization, showing that an improved qualification can be achieved, leading to fast rates for a larger class of learning tasks.

\section{Main Result}\label{sec:results}
    
Our main result establishes an optimal non-asymptotic bias variance decomposition of the excess risk. It is optimal in the sense that choosing an appropriate regularization parameter $\lambda$ enables to achieve the optimal lower rates of convergence established for least squares.

\begin{theorem}[Optimal rates of IT estimator]\label{th:optimal_rates_main_body}
    Let $\delta \in (0, 1]$, and set $\lambda \in (0, \cstLambd{0})$, $n \geq \cstN{}$. The following bound on the excess risk holds with probability greater than $1-2\delta$: 
    \begin{equation}\label{eq:risk_bias_and_var}
        \lossL{}{}{\estEmp{\lambda}{t}} - \lossL{}{}{\optimum} \leq \cstCBias \lambda^{2s}
        + \cstCVar \frac{\df_\lambda}{n}, \text{ with } s = \min \cb{r+1/2, t}.
    \end{equation}
    If we further assume that the capacity condition holds and that the estimator does not saturate, that is $t \geq r+1/2$, then setting
    \begin{equation}
        \lambda = \cstCRisk \; n^{-\frac{\alpha}{1+\alpha(2r+1)}},
    \end{equation}
    makes the following holds with probability greater than $1-2\delta$: 
    \begin{equation}\label{eq:optimal_rate_main_thm}
        \lossL{}{}{\estEmp{\lambda}{t}} - \lossL{}{}{\optimum} \leq 2 \cstCRisk \; n^{-\frac{\alpha(2r +1)}{1+\alpha(2r+1)}}.
    \end{equation}
    
    The constants $\cstLambd{0}, \cstN{}, \cstCBias, \cstCVar, \cstCRisk$ are detailed in  \cref{th:optimal_rates} in the appendix; they are explicit and depend only on $r,\alpha,\cstS,R,t,\delta$ and the distribution $\rho$.
\end{theorem}

\paragraph{Optimal rates.}
First, we note that the decay rate of the excess risk is optimal provided $t \geq r + 1/2$. It means that, up to constant factors, no estimators trained on $n$ observations can benefit from a better learning rate (in the worse case sense) with the prior considered on $\rho$, that is source and capacity conditions of parameters $r, \alpha$. This leads to the second point: we see that IT has qualification $q=t$. When $t=1$, this is Tikhonov estimator and we recover the result of \cite{regularized-erm-gsc}. This qualification shows in the bound on the bias: if $r \leq t - 1/2$, the bias is optimal in $\lambda^{2r+1}$; otherwise, it is suboptimal and decays only in $\lambda^{2t}$, which  leads to higher excess risk, hence generalization error. %

\paragraph{Influence of $t$.}
The leading multiplicative constant of the rate $\cstCVar$ in \cref{eq:risk_bias_and_var} depends linearly on the number of steps $t$, as shown in \cref{eq:explicit_constants} in \cref{sec:optimal_rates_it}. Thus, the rate in \cref{eq:optimal_rate_main_thm} is optimal in $n$ when $t=O(r)$. Letting $t$ go to infinity amounts to minimizing the empirical risk, which yield the unregularized estimator: this agrees with our bound on the excess risk, as the constant $\cstCVar$ would go to infinity in that case.

\paragraph{Source and capacity condition.}
The source and capacity conditions enable precise bounds on the bias and the variance, respectively. If they do not hold, the bias can only be bounded by $O(\lambda)$, while we can upper bound the degrees of freedom with $O(1/\lambda)$, leading to slow learning rates. If the source condition holds but the capacity condition does not, we then obtain learning rates in $n^{-2s/(2s+1)}, s = \min \cb{r+1/2, t}$, which are also optimal in these settings.

\textit{Example: a very easy learning task.} Suppose the source condition satisfies $r = 10$ and that the capacity condition does not hold. Then, using Tikhonov estimator \cite{regularized-erm-gsc} amounts to setting $t=1$. The generalization error would then decay as $n^{-2/3}$. On the other hand, using Iterated Tikhonov estimator with $t = 10$ would make the generalization error decay in $n^{-20/21}$, which is much better.

\subsection{Sketch of the proof}

The proof, which is fully detailed in the appendix, has the following outline:
\begin{itemize}
    \item First, we give technical results on generalized self concordant functions;
    \item Then, we define the intermediate quantity in our bias-variance decomposition;
    \item Finally, we proceed to bounding the bias and the variance separately, which plugged together give our bound on the excess risk.
\end{itemize}

To prove the theorem above we build upon the tools from \cite{regularized-erm-gsc} on generalized self concordant functions. The resulting proof covers and simplifies the case of Tikhonov regularization (one step of iterated Tikhonov) and generalizes the rates to $r > 1/2$. We provide also a fine control of the constants, that takes into account the sequential nature of the IT estimator.

\paragraph{Properties of generalized self concordant loss functions}
Here, we report key properties of GSC loss functions, which are covered in depth in \cref{sec:appendix_settings}.
GSC loss functions are convenient to study as they come with a set of bounds on the Hessian, the gradients and the function values. Intuitively, by integrating multiple times the relation between the third and second derivative in the definition from \cref{eq:def_gsc}, one can obtain bounds on function values. To introduce them, we first define the following function:
\begin{equation}
    \forall \theta \in \HH, \quad \tOp(\theta) = \sup_{z \in \Supp \rho} \sup_{g \in \phi(z)} \absv{g \cdot \theta}.
\end{equation}
By integrating three times the bound of the definition, one can show that: 
\begin{equation}\label{eq:prop_gsc_function_values}
    \lossL{}{}{\estEmp{\lambda}{t}} - \lossL{}{}{\optimum} \leq \Psi\p{\tOp(\estEmp{\lambda}{t} - \optimum)} \norm{\estEmp{\lambda}{t} - \optimum}_{\Hess{}{}{\optimum}}^2, \quad \Psi: t \mapsto (e^t - t - 1)/t^2.
\end{equation}
This type of bound first appeared in \cite{10.1214/09-EJS521} and was given in this form in \cite{regularized-erm-gsc}. We report it in \cref{prop:properties_gsc} in the appendix. For instance, when $\ell$ is the square loss, $\tOp = 0$ everywhere and the r.h.s turns to $1/2 \normtxt{\estEmp{\lambda}{t} - \optimum}_{T}^2$, see \cite{blanchard,JMLR:v18:16-335}. 
On top of this, we generalize a lower bound on the gradient: 
\begin{lemma}[Stacking operator on gradient bounds]\label{lem:stacking_operator_gradient_main_body}
    Let $\theta, \nu, \xi \in \HH$, $\lambda > 0$. 
    If $A: \HH \to \HH$ commutes with $\Hess{}{}{\xi}$, the following holds:
    \begin{equation}\label{eq:prop_gsc_gradient}
        e^{-\tOp\p{\theta - \xi}} \phibot\p{\tOp(\nu - \theta)} \norm{A(\nu - \theta)}_{\Hess{\lambda}{}{\xi}} \leq \norm{A (\nabla \lossL{\lambda}{}{\nu} - \nabla \lossL{\lambda}{}{\theta})}_{\Hess{\lambda}{-1}{\xi}},
    \end{equation}
    where $\phibot: t \mapsto (1 - e^{-t})/t$.
\end{lemma}
Together with \cref{eq:prop_gsc_function_values}, this result is the workhorse of our proof for the upper bound on the excess risk. It is detailed and proven in \cref{sec:technical_lemmas_appendix}. 
\paragraph{Bias-variance decomposition.}
Thanks to \cref{eq:prop_gsc_function_values}, we can relate the excess risk with the distance between estimates. This is why bounding the excess risk amounts to finding a good bias-variance decomposition. Most of the proof we find for the square loss rely on the quantity
\begin{equation}\label{eq:optimal_bv_decomposition_ls}
    \estInt{\lambda}{t} = g_\lambda^t (\hat{T}) \hat{T} \optimum,
\end{equation}
with $\hat{T} = 1/n \sum_i \Phi(x_i) \otimes \Phi(x_i)$ the empirical covariance operator, obtained by replacing $\rho$ with the empirical distribution in \cref{remark:ls}. This is basically the estimator trained on \textit{noiseless empirical data} (\textit{i.e.} using $\optimum(x_i)$ instead of $y_i$) \cite{blanchard, JMLR:v21:20-097,lin2019boosted}. Unfortunately, working with GSC function makes the spectral filtering point of view inapplicable. We need to translate a closed-form expression of the intermediate quantity with filters into the solution of an optimization problem. In our case, we can achieve the optimal bias-variance decomposition with the following quantity:
\begin{equation}\label{eq:optimal_bv_decomposition}
    \begin{aligned}
        \estInt{\lambda}{0} &= \optimum, \\
        \estInt{\lambda}{k+1} &= \prox{\lossLEmp{}{}{}/\lambda} (\estInt{\lambda}{k}), \quad k \geq 0.
    \end{aligned}
\end{equation}
Consequently, we write
\begin{equation}\label{eq:sketch_proof_decomp}
    \normtxt{\estEmp{\lambda}{t} - \optimum}_{\Hess{}{}{\optimum}} \leq \normtxt{\estEmp{\lambda}{t} - \estInt{\lambda}{t}}_{\Hess{}{}{\optimum}} + \normtxt{\estInt{\lambda}{t} - \optimum}_{\Hess{}{}{\optimum}}.
\end{equation}
We recover \cref{eq:optimal_bv_decomposition_ls} with the square loss. In \cite{regularized-erm-gsc}, a different decomposition is used; we found \cref{eq:optimal_bv_decomposition} to greatly simplify the proof. 

\paragraph{Bounding the bias and the variance.}
The first term in \cref{eq:sketch_proof_decomp} is the \textit{bias} of the estimator, as it goes to $0$ when the regularization $\lambda$ goes to $0$. By applying the lower bound on gradient values -- \cref{eq:prop_gsc_gradient} -- with the definition of the proximal operator, one can express $\normtxt{\estEmp{\lambda}{t} - \estInt{\lambda}{t}}$ function of $\normtxt{\estEmp{\lambda}{t-1} - \estInt{\lambda}{t-1}}$. Unfolding the recursion, we obtain \cref{th:bound_on_bias} in the appendix. It shows that the bias decreases in $O\p{\lambda^{r+1/2}}$ if the qualification is sufficient, \textit{i.e.} $t \geq r+1/2$. Otherwise, we recover the saturation experienced with least squares: the bias only decreases in $O\p{\lambda^{t}}$. Specific attention is devoted to bounding the prefactor, which is otherwise difficult to manage. 

The second term in \cref{eq:sketch_proof_decomp} is the \textit{variance}, as it goes to $0$ when the number of samples $n$ increases. \cref{th:bound_on_variance} shows that it decays in $O(\sqrt{\nicefrac{\df_\lambda}{n}})$. It follows closely the work of \cite{blanchard}. However, we cannot use the convenient fact that $\df_\lambda = \Tr \Hess{}{}{\optimum} \HessEmp{\lambda}{-1}{\optimum}$, which is valid for least squares but not in general. Thus, we took specific care in adapting our bounds to the different regimes so as not to impact the learning rate. 

Plugging these results together, we obtain the upper bound on the excess risk. %

\subsection{Optimization}\label{subsec:optim}

The aim of this section is to extend the result of \cref{th:optimal_rates_main_body} to a practical case, where we only have access to an inexact solver for computing the proximal operator. Specifically, let $\epsilon > 0$ be the error (to be defined precisely in \cref{prop:error_propagation_proximal_sequence_main}) made when approximating {$\estEmp{\lambda}{t}$} with {$\estApprox{\lambda}{t}$}, the quantity we compute numerically. We aim for a bound of the type:
\begin{equation*}
    \lossL{}{}{\estApprox{\lambda}{t}} - \lossL{}{}{\optimum} \leq \cstCRisk \; n^{-\frac{\alpha(2r +1)}{1+\alpha(2r+1)}} + \epsilon.
\end{equation*}
The first term in the right hand side is the \emph{statistical error}, and is optimal following the discussion of \cref{th:optimal_rates_main_body}. The second term is the \emph{optimization error}, which is the price to pay for approximating {$\estEmp{\lambda}{k}$} by {$\estApprox{\lambda}{k}$} with tolerance $\epsilon$. The goal is to give a simple optimization rule on the sub-problems to ensure that $\varepsilon$ is of the same order as the upper-bound for the noiseless case.

Assuming that we cannot compute the proximal operator in \cref{eq:it_estimator_definition} exactly, we need to evaluate how the error in approximating {$\estEmp{\lambda}{1}$} propagates to the evaluation of {$\estEmp{\lambda}{2}$}, and so on. As generalized self-concordant functions are well suited to (approximate) second-order optimization scheme, we assume we use a solver with guarantees on a quantity called {\em Newton decrement}, such as the one developed in \cite{NEURIPS2019_60495b4e}. Starting from {$\estEmp{\lambda}{0} = \estApprox{\lambda}{0} = 0$}, define the following for $k>0$:
\begin{align}
    \label{eq:nwt_decrement_theory}
    \estEmp{\lambda}{k} &= \arg\min_{\theta \in \HH} \lossLEmp{\lambda}{k-1}{\theta} \eqdef \lossLEmp{}{}{\theta} + \frac{\lambda}{2} \norm{\theta - \estEmp{\lambda}{k-1}}^2, && \NwtdecTrue{\lambda}{k}{\theta} \eqdef \norm{\nabla \lossLEmp{\lambda}{k-1}{\theta}}_{\HessEmp{\lambda}{-1}{\theta}}, \\
    \label{eq:nwt_decrement_computed}
    \estApprox{\lambda}{k} &\approx \arg\min_{\theta \in \HH} \lossLApprox{\lambda}{k-1}{\theta} \eqdef \lossLEmp{}{}{\theta} + \frac{\lambda}{2} \norm{\theta - \estApprox{\lambda}{k-1}}^2, && \NwtdecApprox{\lambda}{k}{\theta} \eqdef \norm{\nabla \lossLApprox{\lambda}{k-1}{\theta}}_{\HessEmp{\lambda}{-1}{\theta}}.
\end{align}
$\estApprox{\lambda}{k}$ approximates the proximal operator evaluated on $\estApprox{\lambda}{k-1}$, and
$\lossLApprox{\lambda}{t-1}{}$ is the function we manipulate at step $t$. If the optimization was carried without error in \cref{eq:nwt_decrement_computed}, we would have $\lossLApprox{}{t-1}{} = \lossLEmp{}{t-1}{}$.
The quality of the approximation is measured with the Newton decrement of \cref{eq:nwt_decrement_computed}, see \textit{e.g}, Lemma 6 of \cite{NEURIPS2019_60495b4e}.
We need to enforce a bound on the true Newton decrement in \cref{eq:nwt_decrement_theory} when we only have access to $\lossLApprox{\lambda}{t-1}{}$. The next proposition gives a simple rule to achieve this. 

\begin{proposition}[Error propagation with proximal sequence]\label{prop:error_propagation_proximal_sequence_main}
    Let $\epsilon > 0$ the target precision. Assume that we can solve each sub-problem with precision $\bar{\epsilon}_k$:
    \begin{equation*}
        \forall k \in \cb{1, \dots, t}, \quad 
            \NwtdecApprox{\lambda}{k-1}{\estApprox{\lambda}{k}} \leq \bar{\epsilon}_k = \epsilon \frac{1.4^{k-t}}{t},
    \end{equation*}
    and that $\epsilon \leq \sqrt{\lambda}/(2R)$. This suffice to achieve an error $\epsilon$ on the target function:
    \begin{equation*}
        \NwtdecTrue{\lambda}{t-1}{\estApprox{\lambda}{t}} \leq \epsilon.
    \end{equation*}
\end{proposition}
This is a specialized version of \cref{prop:error_propagation_proximal_sequence}, whose proof is detailed in the appendix. Intuitively, this means that enforcing a geometrically higher precision on the first steps is sufficient to obtain high precision on the final estimate. 
To compute IT's estimator in practice, one would need to solve $t$ optimization problem with decreasing precision. As second order schemes have double logarithmic complexity w.r.t the precision $\epsilon$, the complexity of computing the proximal sequence of IT with tolerance $\epsilon$ would be only (up to logarithm term) $t$ times bigger than estimating Tikhonov estimator with tolerance $\epsilon$.
In practice, when learning with kernels, one would use the representer theorem and aim at estimating $\beta$ in $\RR^n$ as in \cref{eq:representer_theorem_and_spectral_filter} \cite{scholkopf2001generalized}. This results in an optimization problem with $n$ observations in dimension $n$, with complexity $O(n^3)$.
A practical implementation could use Nystr\"om projection to avoid this cubic computational burden in the number of samples. The statistical effects of such projection are well studied with Tikhonov regularization \cite{NEURIPS2019_60495b4e, NIPS2015_03e0704b}; their effect on other regularization scheme is an interesting future research direction. 

This proposition can be used directly to bound the excess risk with inexact solvers.
\begin{proposition}[Upper bound on the excess risk with inexact solvers]\label{prop:bound_excess_risk_inexact_solver_main}
    Let $\delta \in (0, 1)$ and assume that the \emph{statistical} assumptions of \cref{th:optimal_rates_main_body} hold as well as the \emph{optimization} assumptions of \cref{prop:error_propagation_proximal_sequence_main}.
    Then, the following bound on the excess risk holds with probability greater than $1-2\delta$:
    \begin{equation}\label{eq:excess_risk_inexact_solver}
         \lossL{}{}{\estApprox{\lambda}{t}} - \lossL{}{}{\optimum} \leq 2 \cstCRisk \; n^{-\frac{\alpha(2r +1)}{1+\alpha(2r+1)}} + \cstE{1/2} \, \epsilon, \quad s = \min\cb{r+1/2,t}
    \end{equation}
    with $\cstCRisk$ as in \cref{th:optimal_rates_main_body} and $\cstE{1/2} \leq 4.3 \cdot 10^3$.
\end{proposition}
This is a specialized version of \cref{prop:bound_excess_risk_inexact_solver} proved in the appendix. The first term is the statistical excess risk, whereas the second term in $\epsilon$ is the price we pay for inexact approximation. For the sake of clarity, crude upper bounds were used (notably {$\HessEmp{\lambda}{-1/2}{\cdot} \leq \cstBTwoStr/\sqrt{\lambda}$}) at the expanse of big constants. They can be expected to be an order of magnitude lower in practice.

\paragraph{Setting $t$ in real application.}
In classical machine learning settings, we do not have access to the source condition parameter $r$. The number of proximal steps $t$ can be seen as an hyperparameter, which is chosen by cross-validation. One would run the algorithm and test the resulting error on a validation set for each iteration, and keep doing proximal steps as long as the validation loss improves.

\section{Experiments}
\label{sec:numerics}
The purpose of the experiments is to illustrate the saturation effect of the Tikhonov estimator when $r \gg 1/2$, and see how the saturation is overcome by iterated Tikhonov $\mathsf{IT}$. We also show that the statistical rates we derive are achieved both in theory and in practice on synthetic data with well-controlled source and capacity conditions.

\paragraph{Settings.}
To that end, we use a synthetic binary classification data set for which we know the source and capacity condition parameters $r$ and $\alpha$ by design. Then, we study the performance of $\IT{t}$, $t \in \cb{1, \dots, 8}$, trained with the logistic loss, which satisfies \cref{def:gsc} about generalized self-concordant functions. Related experiments were conducted in the context of kernel ridge regression with synthetic data in~\cite{NIPS2017_61b1fb3f}, which we follow here. Specifically, we use splines of order $\alpha$ to define a kernel matrix: 
\begin{equation*}
    K(x,z)  =  \Lambda_{\alpha}(x, z)  =  \sum_{k \in \ZZ} \frac{e^{2i\pi k (x-z)}}{\absv{k}^\alpha},
\end{equation*}
for which a closed form expression is available as soon as $\alpha$ is a positive even integer (see for instance Eq (2.1.7) in \cite{Wahba90a}). We then use $\XX = \csb{0, 1}$, $\rho_x$ is the uniform distribution, and $\optimum(x) = \Lambda_{(r+1/2)\alpha+1/2}(0, \cdot)$, which may be shown to live in the RKHS~$\HH$ of~$K$. Then, it is possible to show that the source and capacity assumption are satisfied with value $r, \alpha$, see \cite{NIPS2017_61b1fb3f}.

Finally, we design the distribution $\rho_{y|x}$ of the labels such that $\optimum$ is indeed the minimizer of the risk over $\HH$. This may be ensured if $\optimum$ coincides with the minimizer of the risk over the set of measurable functions, which has the following form under mild assumptions (see Eq. (3) in \cite{JMLR:v21:20-097}):
\begin{equation}\label{eq:logistic_numeric_optimum}
    \optimum(x) = \arg\min_z \EE_{y\mid x}\csb{\ell(y, z)}.
\end{equation}
The previous relation can be satisfied by choosing $\rho_{y\mid x}$ accordingly. More precisely, we need
\begin{equation*}
 \YY = \cb{-1, 1}, \quad \PP(y = 1\mid x) = \p{1 + e^{-\optimum(x)}}^{-1}, \quad \PP(y = -1\mid x) = \p{1 + e^{\optimum(x)}}^{-1},
\end{equation*}
which ensures that \cref{eq:logistic_numeric_optimum} holds -- see details in \cref{sec:numerics_binary}. To our knowledge, this is the first synthetic dataset with given source and capacity condition for classification tasks. For each $\lambda, t$, we sample $n$ points uniformly on $\csb{0, 1}$, evaluate {$\optimum$}, the observed labels $y_i$, and {$\estEmp{\lambda}{t}$}. We evaluate the excess risk {$\lossL{}{}{\estEmp{\lambda}{t}} - \lossL{}{}{\optimum }$} with Monte Carlo sampling. We then report the \textit{lowest excess risk} achieved across the regularization $\lambda$, and the \textit{optimal regularization} used to achieve this loss. We plot lines of slope $\nicefrac{2s \alpha}{1 + 2s \alpha}$ and $\nicefrac{\alpha}{1+ 2s \alpha}$ respectively, with $s = (r+1/2) \land t$ in order to compare the statistical rates achieved in practice and in theory.

\paragraph{Results.}
Results for the logistic loss are available in \cref{fig:logistic_results} and we also present results with least squares where the noise is Gaussian in \cref{sec:numerics_ls_results}.
We set $\alpha = 2$, $r \in \cb{1/4, 41/4}$, and we study the performance of Iterated Tikhonov estimators with $t \in \cb{1, 3, 8}$. $t=1$ corresponds to Tikhonov estimator and saturates at $r=1/2$. $\IT{3}$ and $\IT{8}$ saturates at $r=5/2$ and $r=15/2$ respectively. Consequently, all estimators have optimal rates on the difficult task with $r=1/4$; however, only IT exploits the additional regularity of the easy task, with $r=41/4$. 
This experimentally shows that better sample complexity can be achieved when the learning task is easier and $t$ is high, matching the rates predicted in \cref{th:optimal_rates}, which are $n^{-\nicefrac{\alpha(1 + 2s)}{1 + \alpha(1 + 2s)}}$, with $s = \min\cb{r, t-1/2}$. Learning rates were estimated with an ordinary least square regression in log-log scale, and are given in \cref{tab:learning_coefficients}, where they are compared with the theoretical values.
To conclude, we observe a slight improvement in absolute value of the excess risk in the range $r \ll t$, suggesting that IT is useful even when the learning task is hard. This could be because of lower constants for high $t$: \textit{e.g.} we show that $\cstCBias$ decays in $1/t^{r}$ when $t \geq r+1/2$, see \cref{th:bound_on_bias} in the appendix. We report in the appendix additional experimental results such as plots with the chosen regularization $\lambda$ as a function of $n$, and plots on the ratio between the excess risk of $\IT{t}$ and Tikhonov, to show that the former is consistently better than the latter on easy tasks. 

\begin{figure}[t]
    \centering
    ~~~\begin{subfigure}[t]{0.45\linewidth}
        \centering
        \includegraphics[width=\linewidth]{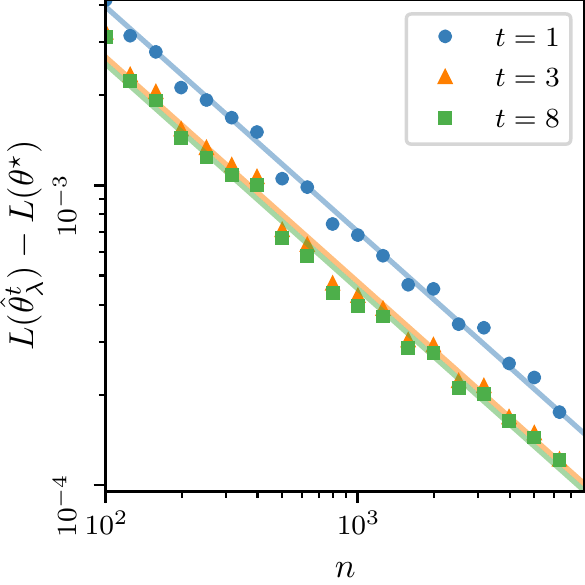}
    \end{subfigure}
    \hfill
    \begin{subfigure}[t]{0.45\linewidth}
        \centering
        \includegraphics[width=\linewidth]{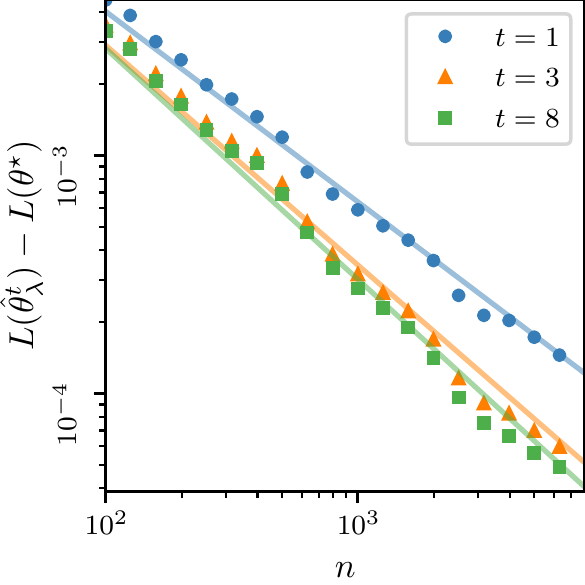}
    \end{subfigure}~~~
    \caption{Excess risk for various Iterated Tikhonov estimators as a function of $n$. \textbf{Colors}: $t=1$ (Tikhonov) estimator is shown in blue; $t=3, 8$ in green, orange. \textbf{Left}: from a difficult problem, $r=1/4, \alpha=2$. \textbf{Right}: easy problem, $r=41/4, \alpha=2$. Plain lines are predicted by theory, with slope $-\nicefrac{\alpha(1 + 2s)}{1 + \alpha(1 + 2s)}$, $s = \min\cb{r, t-1/2}$ (see main text). All plots are averaged over $100$  runs of the optimization procedure with different initialization.}
    \label{fig:logistic_results}
\end{figure}

\begin{table}[ht]
\centering
\caption{Learning rate coefficients for capacity condition $\alpha =2$ and various source condition assumption $r$. We estimate $\gamma$ with ordinary least square with the model {$L(\estEmp{\lambda}{t}) - L(\optimum) \propto n^{-\gamma}$}. We display the coefficient we expect in theory, and the one we estimate.}
\label{tab:learning_coefficients}
\begin{tabular}{@{}ccccc@{}}
\toprule
 & $r$ & $0.25$ & $3.25$ & $10.25$ \\ \midrule
\multirow{2}{*}{$t=1$} & Theory & $0.75$ & $0.80$ & $0.80$ \\
 & Estimation & $0.71$ & $0.73$ & $0.72$ \\ \midrule
\multirow{2}{*}{$t=3$} & Theory & $0.75$ & $0.92$ & $0.92$ \\
 & Estimation & $0.75$ & $0.83$ & $0.87$ \\ \midrule
\multirow{2}{*}{$t=8$} & Theory & $0.75$ & $0.94$ & $0.97$ \\
 & Estimation & $0.79$ & $0.95$ & $0.98$ \\ \bottomrule
\end{tabular}
\end{table}

\section{Conclusion}
This paper studies a well-known regularization scheme for least square, and extend it for the first time to other loss functions, which notably contain the logistic loss used for classification.
We prove that Iterated Tikhonov, corresponding to proximal point iterations, has optimal learning rates and higher qualification than Tikhonov, and as such could outperform it on easy tasks. We extend the scope of the theory of learning with generalized self concordant loss functions beyond standard Tikhonov regularization, which fills a gap in the previous theory, showing that it is possible to be fully adaptive to the regularity of the learning problem, without saturation effects.
On top of this, we gave sufficient conditions to compute the estimator in practice, which is nontrivial by its sequential nature. Interesting research directions include related regularization schemes, such as boosting, but also implementations of the iterated Tikhonov procedure with sketching techniques as Nystr\"om projections. The goal is to derive algorithms  that are both optimal, in terms of statistical guarantees, and with reduced computational complexity, which is an aspect we will address in future work.

\subsection*{Acknowledgments}
A.R. acknowleges support of the French government under management of Agence Nationale de la Recherche as part of the “Investissements d’avenir” program, reference ANR-19-P3IA-0001 (PRAIRIE 3IA Institute). A.R. acknowledges support of the European Research Council (grant REAL 947908).  J. Mairal was supported by the ERC grant number 714381 (SOLARIS project) and by ANR 3IA MIAI@Grenoble Alpes, (ANR19-P3IA-0003).

\bibliographystyle{unsrt}
\bibliography{mybib}

\newpage
\appendix
\doparttoc
\faketableofcontents 
\addcontentsline{toc}{section}{Appendix} %
\part{Appendix} %
\parttoc %

\section{Settings, notations and assumptions}
\label{sec:appendix_settings}

Given a separable Hilbert space $\HH$, $\norm{\cdot}$ denotes the norm in $\HH$. For any operator $A$ on $\HH$, $\norm{A}$ denotes its operator norm, and $\Tr A$ its trace norm. If $A$ is a p.d operator, we denote by $\norm{\cdot}_A = \norm{A^{1/2} \cdot}$ the norm induced by $A$. We denote $\norm{A}_{HS}$ the Hilbert Schmidt norm of $A$. We use the short-hand notation
\begin{equation*}
    A_\lambda = A + \lambda \II,
\end{equation*}
where $\II$ is the identity. We denote by $a \land b$ the minimum of $\cb{a, b}$, and $a \lor b$ its maximum.

\subsection{Settings and technical assumptions}\label{sec:settings_apendix}
The settings in this subsection are the same as in \cite{regularized-erm-gsc}. We report them for completeness. 

Let $\XX$ a Borel input space, $\YY$ be a vector-valued output spaces, and
$\rho$ a probability distribution on $\XX \times \YY$. We consider $\HH$ to be a separable Hilbert space of functions from $\XX$ to $\YY$. We consider a loss function $\ell: \YY \times \YY \to \RR$ for measuring the fit between predictions and true labels. Given $n$ observations $(x_1, y_1), \dots, (x_n, y_n)$ i.i.d according to $\rho$, the goal is to build a measurable function $\estEmp{}{}$, which minimizes the expected loss
\begin{equation*}
    L(\estEmp{}{}) = \EE_{x, y \sim \rho} \csb{\ell(y, \estEmp{}{}(x))}.
\end{equation*}
In this paper, we evaluate the quality of the estimator with probabilistic upper bounds on the \textit{excess risk}
\begin{equation*}
    L(\estEmp{}{}) - \inf_{\theta \in \HH} L(\theta) \leq K n^{-\gamma},
\end{equation*}
with probability greater than $1 - \delta$. The rate of decay $\gamma$ is referred to as the \textit{learning rate} of the estimator. 
Our main assumption on the loss function is to be generalized self-concordant (GSC).
\begin{assumption}[Generalized Self-Concordance]\label{hyp:gsc_loss}
    For any $z = x, y \in \XX \times \YY$, the function $\ell_z : \HH \to \RR$ defined as $\ell_z(\theta) = \ell(y, \theta(x))$ for $\theta \in \HH$ is convex and three times differentiable. Besides, there exists a set $\phi(z) \subset \HH$ s.t
    \begin{equation*}
        \forall \theta \in \HH, \, \forall h,k \in \HH, \quad \absv{\nabla^3\ell_z(\theta) \csb{h, k, k}} \leq \sup_{g \in \phi(z)} \absv{k \cdot g} \nabla^2 \ell_z (\theta) \csb{k, k}.
    \end{equation*}
\end{assumption}

Next, we introduce the following quantities.
\begin{definition}[Useful quantities]\label{def:all_quantities}
    Let $\theta \in \HH$. The following quantities are independant of the random variable $z \sim \rho$, either by taking the supremum over the support of $\rho$ or by considering the expectation. Define:
    \begin{itemize}
        \item uniform bounds on the derivatives:
        \begin{equation*}
            \cstBOne(\theta) = \sup_{z \in \Supp \rho} \norm{\nabla \ell_z(\theta)}, \quad \cstBTwo(\theta) = \sup_{z \in \Supp \rho} \Tr \nabla^2 \ell_z(\theta);
        \end{equation*}
        \item the Hessian of the expected and empirical loss:
        \begin{equation*}
            \Hess{}{}{\theta} = \nabla^2 \EE\csb{\ell_z(\theta)}, \quad \HessEmp{}{}{\theta} = \frac{1}{n} \sum_{i=1}^n \nabla^2 \ell_{z_i}(\theta);
        \end{equation*}
        \item the function $\tOp$, s.t:
        \begin{equation*}
            \tOp(\theta) = \sup_{z \in \Supp \rho} \sup_{g \in \phi(z)} \absv{\theta \cdot g}.
        \end{equation*}
    \end{itemize}
\end{definition}

We make technical assumption to ensure that the loss function and its derivatives are well defined everywhere and that we can exchange expectation and derivative. 
\begin{assumption}[Technical assumptions]\label{hyp:technical_assumptions}
    There exists $R$ s.t $\sup_{g \in \phi(z)} \norm{g} \leq R$ almost surely; $\absv{\ell_z(0)}, \norm{\nabla \ell_z(0)}, \Tr \Hess{}{}{\nabla^2 \ell_z(0)}$ are almost surely bounded.
\end{assumption}
Using Prop. 2 of \cite{regularized-erm-gsc}, we have that $\cstBOne(\theta), \cstBTwo(\theta), \lossL{}{}{\theta}, \nabla \lossL{}{}{\theta}, \Hess{}{}{\theta}$ exist for all $\theta \in \HH$, and
\begin{equation*}
    \nabla \lossL{}{}{\theta} = \EE\csb{\nabla \ell_z (\theta)}, \quad \Hess{}{}{\theta} = \EE\csb{\nabla^2 \ell_z(\theta)}.
\end{equation*}
Finally, $\Hess{}{}{\theta}$ is trace-class, that is its trace is finite for any $\theta \in \HH$. The same properties hold when considering $\hat{\rho}$ instead of $\rho$, that is for the quantities $\lossLEmp{}{}{\theta}, \nabla \lossLEmp{}{}{\theta}$ and $\HessEmp{}{}{\theta}$.

We make three key assumptions to obtain our learning rate.
\begin{assumption}[Existence of a minimizer]\label{hyp:existence_minimizer_appendix}
    There exists a minimizer of $\lossL{}{}{}$ in $\HH$. There is $\optimum \in \HH$ s.t
    \[
    \lossL{}{}{\optimum} = \inf_{\theta \in \HH} \lossL{}{}{\optimum}.
    \]
\end{assumption}
\begin{assumption}[Source condition]\label{hyp:source_appendix}
    There exists $r > 0$ and $v \in \HH$ s. t
    \[
    \optimum = \Hess{}{r}{\optimum} v.
    \]
\end{assumption}
The third assumption qualifies the ill-posedness of the problem:
\begin{assumption}[Capacity condition]\label{hyp:capacity_appendix}
    There exists $\alpha >1$, $\cstSsmall, \cstS > 0$ s.t
    \[
    \cstSsmall \lambda^{-1/\alpha} \leq \df_\lambda \leq \cstS \lambda^{-1/\alpha}. 
    \]
\end{assumption}
To understand the source and capacity condition, one must pay attention to the counterpart of the covariance operator for GSC loss function, that is the expected hessian at optimality. It is denoted with $\Hess{}{}{\optimum}$ throughout the paper. The source and capacity conditions are assumptions on the eigendecomposition of this operator. To better quanitfy these assumptions, take $\sigma_j, e_j$ an eigenbasis of $\Hess{}{}{\optimum}$, with $\sigma_j > \sigma_{j+1}$. 

The source condition is a smoothness assumption on $\optimum$. It amounts to assuming that the eigendecomposition of $\optimum$ on the basis of the Hessian decays faster than its spectrum. Indeed, rewriting \cref{hyp:source_appendix} we obtain
\begin{equation*}
    \norm{v}^2 = \sum_{j \geq 1} \sigma_j^{-2r} \dotprod{\optimum}{e_j}^2 < + \infty.
\end{equation*}
Assuming $r =0$ simplifies to $\optimum \in \HH$. Bigger $r$ implies that the optimum can be well approximated by the first few eigenvectors (as $(\sigma_j^{-2r})_j$ goes quickly to infinity). 

Similarly, the capacity condition is an assumption on the decay of the spectrum of the Hessian. Specifically, it assumes that the spectrum decays polynomially, i.e $\sigma_j \sim j^{-\alpha}$. As this operator is compact, we have $\alpha > 1$ for the $\sum_j j^{-\alpha}$ to be summable. Bigger $\alpha$ gives easier input space $\XX$. 

See \cref{sec:settings} in the main body of the paper for a discussion on the significance of these assumptions.

\subsection{Basic results on GSC loss functions}
Here, we present Prop. 4 of \cite{regularized-erm-gsc}, which we then extend with an additional lemma. 
\begin{proposition}[Properties of GSC functions]\label{prop:properties_gsc}
    Let $\theta, \nu \in \HH$, $\lambda \geq 0$. The following properties hold:
    \begin{align}
        \label{eq:lower_bound_Hessian}
        \HessEmp{\lambda}{}{\theta} &\preceq e^{\tOp(\theta - \nu)} \HessEmp{\lambda}{}{\nu} \\
        \norm{\nabla \lossLEmp{\lambda}{}{\theta} - \lossLEmp{\lambda}{}{\nu}}_{\HessEmp{\lambda}{-1}{\theta}} &\leq \norm{\theta - \nu}_{\HessEmp{\lambda}{}{\theta}} \phitop \p{\tOp(\theta - \nu)} \\
        \label{eq:gsc_upper_bound_function_value}
        \lossL{\lambda}{}{\theta} - \lossL{\lambda}{}{\nu} - \nabla \lossL{\lambda}{}{\nu} \cdot (\theta - \nu) &\leq \Psi\p{\tOp(\theta - \nu)} \norm{\theta - \nu}_{\Hess{\lambda}{}{\theta}}^2
    \end{align}
    where $\phibot: t \mapsto (1 - e^{-t})/t$ and $\Psi: t \mapsto (e^t - t - 1)/t^2$. Moreover, if $\nu, \xi \in \HH$, $A: \HH \to \HH$ commutes with $\Hess{}{}{\xi}$, then the following holds:
    \begin{equation}
        \label{eq:prop_gsc_gradient_lowerbound}
        e^{-\tOp\p{\theta - \xi}} \phibot\p{\tOp(\nu - \theta)} \norm{A(\nu - \theta)}_{\Hess{\lambda}{}{\xi}} \leq \norm{A (\nabla \lossL{\lambda}{}{\nu} - \nabla \lossL{\lambda}{}{\theta})}_{\Hess{\lambda}{-1}{\xi}}
    \end{equation}
\end{proposition}
We slightly modify the lower bound gradient, which is crucial for obtaining higher qualification with IT. 
\begin{lemma}[Stacking operator on gradient bounds]\label{lem:stacking_operator_gradient}
    Let $\theta, \nu, \xi \in \HH$, $\lambda > 0$. If $A: \HH \to \HH$ commutes with $\Hess{}{}{\xi}$, the following holds:
    \begin{equation*}
        e^{-\tOp\p{\theta - \xi}} \phibot\p{\tOp(\nu - \theta)} \norm{A(\nu - \theta)}_{\Hess{\lambda}{}{\xi}} \leq \norm{A (\nabla \lossL{\lambda}{}{\nu} - \nabla \lossL{\lambda}{}{\theta})}_{\Hess{\lambda}{-1}{\xi}}.
    \end{equation*}
\end{lemma}
\begin{proof}
    Defining $v_s = \theta + s(\nu - \theta)$ for $s \in \cb{0, 1}$, we have:
    \begin{align*}
        A^2 \p{\nabla \lossL{\lambda}{}{\nu} - \nabla \lossL{\lambda}{}{\theta}} 
        &= A^2 \int_{0}^{1} \Hess{\lambda}{}{v_s} \p{\nu-\theta} \dd s, \\
        \text{which implies} \quad \dotprod{A^2 \p{\nabla \lossL{\lambda}{}{\nu} - \nabla \lossL{\lambda}{}{\theta}}}{\nu - \theta}
        &= A^2 \int_{0}^{1} \dotprod{\Hess{\lambda}{}{v_s} \p{\nu - \theta}}{\nu - \theta} \dd s.
    \end{align*}
    We may then use the lower bound on the Hessian from \cref{eq:lower_bound_Hessian},
    \begin{align*}
        \Hess{\lambda}{}{v_s} \succeq \Hess{\lambda}{}{\xi} e^{-\tOp\p{v_s - \xi}} \succeq \Hess{\lambda}{}{\xi} e^{-\tOp\p{\theta - \xi}} e^{-s\tOp\p{\nu - \theta}},
    \end{align*}
    where the second inequality comes from $\tOp$ satisfying the triangle inequality.
    Plugging this in the previous equation and using the fact that $\Hess{}{}{\xi}$ and $A$ commute, we have that:
    \begin{equation*}
        \begin{aligned}
            \int_{0}^{1} \dotprod{A^2 \Hess{\lambda}{}{v_s}\p{\nu - \theta}}{\nu - \theta} \dd s 
            &\geq e^{-\tOp\p{\theta - \xi}} \int_0^1 e^{-s\tOp\p{\nu - \theta}} \dd s \; \dotprod{\Hess{\lambda}{}{\xi} A \p{\nu - \theta}}{A\p{\nu - \theta}} \\
            &= e^{-\tOp\p{\theta - \xi}} \phibot\p{\tOp(\nu - \theta)} \dotprod{\Hess{\lambda}{}{\xi}A(\nu - \theta)}{A(\nu - \theta)},
        \end{aligned}
    \end{equation*}
    which gives the lower bound
    \begin{equation}\label{eq:proof_lemma_ineq_1}
        e^{-\tOp\p{\theta - \xi}} \phibot\p{\tOp(\nu - \theta)} \norm{A(\nu - \theta)}_{\Hess{\lambda}{}{\xi}}^2 \leq \dotprod{A^2 \p{\nabla \lossL{\lambda}{}{\nu} - \nabla \lossL{\lambda}{}{\theta}}}{\nu - \theta}.
    \end{equation}
    On the other hand, with Cauchy Schwartz inequality, we obtain:
    \begin{equation}\label{eq:proof_lemma_ineq_2}
        \dotprod{A^2 (\nabla \lossL{\lambda}{}{\nu} - \nabla \lossL{\lambda}{}{\theta})}{\nu - \theta}
        \leq \norm{A (\nabla \lossL{\lambda}{}{\nu} - \nabla \lossL{\lambda}{}{\theta})}_{\Hess{\lambda}{-1}{\xi}} \norm{A (\nu - \theta)}_{\Hess{\lambda}{}{\xi}}.
    \end{equation}
    Combining the inequalities \cref{eq:proof_lemma_ineq_1,eq:proof_lemma_ineq_2} and dividing by $\norm{A(\nu - \theta)}_{\Hess{\lambda}{}{\xi}}$, we obtain the result needed. 
\end{proof}

\section{Proof of Theorem 1}\label{sec:proof_upper_rate_appendix}

\subsection{Error decomposition}

Thanks to \cref{eq:gsc_upper_bound_function_value}, the excess risk is bounded by the distance between estimate in $\Hess{}{}{\optimum}$ norm with
\begin{equation*}
    \lossL{\lambda}{}{\estEmp{\lambda}{t}} - \lossL{\lambda}{}{\optimum} \leq \Psi\p{\tOp(\estEmp{\lambda}{t} - \optimum)} \norm{\estEmp{\lambda}{t} - \optimum}_{\Hess{\lambda}{}{\optimum}}^2.
\end{equation*}
In order to compute $\normtxt{\estEmp{\lambda}{t} - \optimum}_{\Hess{}{}{\optimum}}$, we need to go through an intermediate quantity $\estInt{}{}$. In the context of least squares and spectral filters, such quantity is usually defined to be
\begin{equation}
    \label{eq:optimal_decomposition_LS}
    \estInt{}{} = g_\lambda(\hat{T}) \hat{S}^* \hat{S} \optimum,
\end{equation}
where:
\begin{itemize}
    \item $\hat{T} = \hat{S}^* \hat{S}$ is the \textit{empirical covariance operator}, equal to $\sum_{i=1}^n \Psi(x_i) \otimes \Psi(x_i)$ when $\HH$ is a RKHS with feature map $\Psi$ (see \cref{remark:ls});
    \item $\hat{S}: \HH \to \RR^n$ is the \textit{sampling operator}, with $\hat{S} \theta = 1/\sqrt{n} (\theta(x_i), \dots, \theta(x_n))$;
    \item Its dual is $\hat{S}^*: \RR^n \to \HH$, with $\hat{S}^* y = 1/\sqrt{n} \sum_{i=1}^n y_i \Phi(x_i)$;
\end{itemize}
 see \cite{blanchard} for details. Thus, the quantity in \cref{eq:optimal_bv_decomposition_ls} can be seen as the estimator trained on the \textit{empirical noiseless distribution}, where we use $\hat{S} \optimum$ instead of $y = (y_i)_{1\leq i \leq n}$. It is optimal in the sense that its bias $\norm{\estInt{}{} - \optimum}_{\hat{T}}$ will be of the order of $\lambda^{r + 1/2}$ and its variance $\norm{\estEmp{\lambda}{t} - \estInt{}{}}_{\hat{T}}$ of the order of $\df_\lambda/n$, leading to the optimal rates for least squares \cite{optimal_rate_rls}. 

Expressing the quantity above as a proximal sequence is the key insight of the proof. It turns out that the following quantity obtains the same optimal decomposition.
\begin{definition}[Error decomposition]\label{def:error_decomposition}
    Define the following quantity:
    \begin{equation*}
        \begin{aligned}
            \estInt{\lambda}{0} &= \optimum \\
            \estInt{\lambda}{k+1} &= \prox{\lossLEmp{}{}{}/\lambda} (\estInt{\lambda}{k}), \quad k \geq 0
        \end{aligned}
    \end{equation*}
\end{definition}

\begin{remark}
    In fact, the estimator above, when expressed with filters, has its (bias, variance) equals to the (variance, bias) of the estimator of \cref{eq:optimal_decomposition_LS}. It is easy to change the intermediate quantity of \cref{def:error_decomposition} to match, but it introduces unnecessary burden with the notations.
\end{remark}

The purpose of next sections is to bound
\begin{equation}
    \label{eq:decomposition_distance_estimates}
    \norm{\estEmp{\lambda}{t} - \optimum}_{\Hess{}{}{\optimum}} \leq \norm{\estEmp{\lambda}{t} - \estInt{\lambda}{t}}_{\Hess{}{}{\optimum}} + \norm{\estInt{\lambda}{t} - \optimum}_{\Hess{}{}{\optimum}}.
\end{equation}
The first term will be the \textit{bias} of the estimator (decreases with $\lambda/t$) while the second one will be the \textit{variance} (decreases with $t/\lambda$ and $n$). The intermediate quantity of \cref{def:error_decomposition} being very close to the one of \cref{eq:optimal_decomposition_LS} used in \cite{blanchard}, it is natural that the proof follows similarly. 

\subsection{Bounding the bias}
Here, we proceed in bounding the bias, that is the quantity $\normtxt{\estEmp{\lambda}{t} - \estInt{\lambda}{t}}_{\Hess{}{}{\optimum}}$. 

\begin{theorem}[Improved qualification of Iterated Tikhonov estimator]\label{th:bound_on_bias}
    Let $\delta \in (0, 1]$. Recall the source condition of parameter $r, \norm{v}$. Define the following conditions on the number of samples:
    \begin{align*}
        \Hyp{1}: n &\geq 24 \frac{\cstBTwoStr}{\lambda} \log\frac{16 \cstBTwoStr}{\lambda\delta}, \\
        \Hyp{1b}: n &\geq 8 \frac{\cstBTwoStr^2}{\lambda^2} \log^2 \frac{4}{\delta}, \\
        \Hyp{2}: n &\geq 2 \csb{1 \lor \p{\frac{2 \cstBTwoStr (t-1/2)^r}{\lambda^{s-1/2}}}^2} \log \frac{4}{\delta},
    \end{align*}
    Now assume:
    \begin{align*}
        \Hyp{1} \quad &\iftext r \leq 1/2, \\
        \Hyp{1} + \Hyp{1b} \quad &\iftext 1/2 < r \leq 1, \\
        \Hyp{1} + \Hyp{2} \quad &\iftext r > 1.
    \end{align*}
    Then, with probability greater than $1 - \delta$:
    \begin{equation}
        \norm{\estEmp{\lambda}{t} - \estInt{\lambda}{t}}_{\Hess{}{}{\optimum}}
        \leq \sqrt{2} \cstT(r,t) \cstP{t} \lambda^s,
    \end{equation}
    with $s = (r+1/2) \land t$,
    \begin{equation}
        \cstT(r, t) = \begin{cases}
            \norm{v} (1 \lor (\cstBTwoStr + \lambda)) 2^r            \quad &\iftext r \leq 1, \\
            \norm{v} \frac{w(r) + r}{(t-1/2)^r}                      \quad &\iftext r > 1 \andtext r + 1/2 < t \\
            \norm{v} \frac{w(r)}{(t-1/2)^r} + \cstBTwoStr^{r-t+1/2}  \quad &\iftext r > 1 \andtext r + 1/2 \geq t,
        \end{cases}
    \end{equation}
    $w(r) = r 2^{\intpart{r}+1} \cstBTwoStr^r$, and:
    \begin{equation*}
        \cstP{t} \eqdef \prod_{k=1}^t \phibot^{-1}\p{\tOpEmp(\estEmp{\lambda}{k} - \estInt{\lambda}{k})} e^{\tOpEmp(\estInt{\lambda}{k} - \optimum)}.
    \end{equation*}
\end{theorem}

This term is the optimal bias for LS with the usual excess risk decomposition. The saturation effect is explicit: we go from a bias decay in $\lambda^t$ when $t\leq r+1/2$ to $\lambda^r$ when the source condition saturates IT's regularization. That is, IT's estimator has a qualification of $t$, in the sense that it can exploits source condition up to $r = t-1/2$. If $r > t-1/2$, the estimator saturates and the learning rate becomes suboptimal. 

\begin{proof}
    This proof simply relies on the upper bound on gradients enabled by GSC functions. We will use \cref{lem:stacking_operator_gradient} for that purpose. Also, we will use the definition of a proximal sequence; that is, we have that
    \begin{equation*}
        \forall k \leq t, \quad \nabla \lossLEmp{}{}{\estEmp{\lambda}{k}} + \lambda (\estEmp{\lambda}{k} - \estEmp{\lambda}{k-1}) = 0,
    \end{equation*}
    which is just another way of saying that we perform implicit gradient steps of size $1/\lambda$. 

    ~\paragraph{Changing the norm.}
    We first change the norm we operate on:
    \begin{align*}
        \norm{\estEmp{\lambda}{t} - \estInt{\lambda}{t}}_{\Hess{}{}{\optimum}}
        &\leq \norm{\HessEmp{\lambda}{-1/2}{\optimum}\Hess{}{1/2}{\optimum}}\norm{\estEmp{\lambda}{t} - \estInt{\lambda}{t}}_{\HessEmp{\lambda}{}{\optimum}} \\
        &\leq \norm{\HessEmp{\lambda}{-1/2}{\optimum}\Hess{\lambda}{1/2}{\optimum}}\norm{\estEmp{\lambda}{t} - \estInt{\lambda}{t}}_{\HessEmp{\lambda}{}{\optimum}}.
    \end{align*}
    We bound the operator norm using \cref{prop:main_concentration_bound_operators} in \cref{sec:technical_lemmas_appendix}, with $\FFl = \cstBTwoStr/\lambda$. We obtain:
    \begin{equation}\label{eq:bound_bias_glue_1}
        \Hyp{1}: n \geq 24 \frac{\cstBTwoStr}{\lambda} \log\frac{8 \cstBTwoStr}{\lambda\delta} \implies \norm{\HessEmp{\lambda}{-1/2}{\optimum}\Hess{\lambda}{1/2}{\optimum}} \leq \sqrt{2}.
    \end{equation}
    We now proceed in bounding the distance between estimates, that is the quantity $\normtxt{\estEmp{\lambda}{t} - \estInt{\lambda}{t}}_{\HessEmp{\lambda}{}{\optimum}}$. We denote
    \begin{equation}
        s = (r+1/2) \land t.
    \end{equation}

    ~\paragraph{Upper bound on gradients.}
    Use \cref{lem:stacking_operator_gradient} on $\lossLEmp{\lambda}{}{}$ to have:
    \begin{align*}
        \norm{\estEmp{\lambda}{t} - \estInt{\lambda}{t}}_{\HessEmp{\lambda}{}{\optimum}}
        &\leq \phibot^{-1}\p{\tOpEmp(\estEmp{\lambda}{t} - \estInt{\lambda}{t})} e^{\tOpEmp(\estInt{\lambda}{t} - \optimum)} \norm{\nabla \lossLEmp{\lambda}{}{\estEmp{\lambda}{t}} - \nabla \lossLEmp{\lambda}{}{\estInt{\lambda}{t}}}_{\HessEmp{\lambda}{-1}{\optimum}} \\
        &= \phibot^{-1}\p{\tOpEmp(\estEmp{\lambda}{t} - \estInt{\lambda}{t})} e^{\tOpEmp(\estInt{\lambda}{t} - \optimum)} \norm{\lambda (\estEmp{\lambda}{t-1} - \estInt{\lambda}{t-1})}_{\HessEmp{\lambda}{-1}{\optimum}} \\
        &= \phibot^{-1}\p{\tOpEmp(\estEmp{\lambda}{t} - \estInt{\lambda}{t})} e^{\tOpEmp(\estInt{\lambda}{t} - \optimum)} \norm{\lambda \HessEmp{\lambda}{-1}{\optimum}(\estEmp{\lambda}{t-1} - \estInt{\lambda}{t-1})}_{\HessEmp{\lambda}{}{\optimum}}.
    \end{align*}
    Let us detail the recursion. Let $k \leq t$. Then, the following inequality holds, thanks to \cref{lem:stacking_operator_gradient}:
    \begin{align*}
        \norm{\lambda^k \HessEmp{\lambda}{-k}{\optimum}(\estEmp{\lambda}{t-k} - \estInt{\lambda}{t-k})}_{\HessEmp{\lambda}{}{\optimum}} 
        &\leq \begin{aligned}[t]
        &\phibot^{-1}\p{\tOpEmp(\estEmp{\lambda}{t-k} - \estInt{\lambda}{t-k})} e^{\tOpEmp(\estInt{\lambda}{t-k} - \optimum)} \\    
        &\norm{\lambda^k \HessEmp{\lambda}{-k}{\optimum}\p{\nabla \lossLEmp{\lambda}{}{\estEmp{\lambda}{t-k}} - \nabla \lossLEmp{\lambda}{}{\estInt{\lambda}{t-k}}}}_{\HessEmp{\lambda}{-1}{\optimum}}
        \end{aligned} \\
        &= \begin{aligned}[t]
        &\phibot^{-1}\p{\tOpEmp(\estEmp{\lambda}{t-k} - \estInt{\lambda}{t-k})} e^{\tOpEmp(\estInt{\lambda}{t-k} - \optimum)} \\    
        &\norm{\lambda^{k+1} \HessEmp{\lambda}{-k}{\optimum}\p{\estEmp{\lambda}{t-(k+1)} - \estInt{\lambda}{t-(k+1)}}}_{\HessEmp{\lambda}{-1}{\optimum}}
        \end{aligned} \\
        &= \begin{aligned}[t]
        &\phibot^{-1}\p{\tOpEmp(\estEmp{\lambda}{t-k} - \estInt{\lambda}{t-k})} e^{\tOpEmp(\estInt{\lambda}{t-k} - \optimum)} \\    
        &\norm{\lambda^{k+1} \HessEmp{\lambda}{-(k+1)}{\optimum}\p{\estEmp{\lambda}{t-(k+1)} - \estInt{\lambda}{t-(k+1)}}}_{\HessEmp{\lambda}{}{\optimum}}.
        \end{aligned}
    \end{align*}
    Thus, unfolding the recursion, we obtain:
    \begin{equation}
        \begin{aligned}
            \norm{\estEmp{\lambda}{t} - \estInt{\lambda}{t}}_{\HessEmp{\lambda}{}{\optimum}} &\leq \cstP{t} \norm{\lambda^t \HessEmp{\lambda}{-t}{\optimum}\optimum}_{\HessEmp{\lambda}{}{\optimum}}, \\
            \text{with} \quad \cstP{t} &\eqdef \prod_{k=1}^t \phibot^{-1}\p{\tOpEmp(\estEmp{\lambda}{k} - \estInt{\lambda}{k})} e^{\tOpEmp(\estInt{\lambda}{k} - \optimum)}.
        \end{aligned}
    \end{equation}
    We now use the source condition on $\optimum$. Recall that it gives
    \begin{equation*}
        \optimum = \Hess{}{r}{\optimum} v,
    \end{equation*}
    for some $v \in \HH$. Thus, we have:
    \begin{align}
        \norm{\lambda^t \HessEmp{\lambda}{-t}{\optimum}\optimum}_{\HessEmp{\lambda}{}{\optimum}} 
        &= \norm{\lambda^t \HessEmp{\lambda}{-(t-1/2)}{\optimum}\Hess{}{r}{\optimum} v} \\
        \label{eq:bound_bias_glue_2}
        &\leq \norm{\lambda^t \HessEmp{\lambda}{-(t-1/2)}{\optimum}\Hess{}{r}{\optimum}} \norm{v}
    \end{align}
    We need to distinguish between $r \leq 1$ and $r > 1$ to bound the operator norm
    \begin{equation*}
        \norm{\lambda^t \HessEmp{\lambda}{-(t-1/2)}{\optimum}\Hess{}{r}{\optimum}}.
    \end{equation*}

    ~\paragraph{Case $r \leq 1$.}
    We use the following decomposition:
    \begin{align*}
        \norm{\lambda^t \HessEmp{\lambda}{-(t-1/2)}{\optimum}\Hess{}{r}{\optimum}}
        &\leq \norm{\lambda^t \HessEmp{\lambda}{-(t-1/2)}{\optimum}\HessEmp{\lambda}{r}{\optimum}} \norm{\HessEmp{\lambda}{-r}{\optimum} \Hess{}{r}{\optimum}}.
    \end{align*}
    The first term is bounded like this:
    \begin{align*}
        \norm{\lambda^t \HessEmp{\lambda}{-(t-1/2)+r}{\optimum}} 
        &\leq \sup_{\widehat{\sigma}_{\min} < \sigma \leq \cstBTwoStr} \frac{\lambda^t}{(\sigma+\lambda)^{t-1/2-r}} \\
        &\leq \lambda^s \begin{cases}
            1 \quad &\iftext r+1/2 < t \\
            \cstBTwoStr + \lambda \quad &\iftext t=1 \andtext r>1/2.
        \end{cases}
    \end{align*}
    This illustrates that Tikhonov regularization ($t=1$) saturates at $r=1/2$. 
    
    For the second term, write
    \begin{equation*}
        \norm{\HessEmp{\lambda}{-r}{\optimum} \Hess{}{r}{\optimum}} \leq \norm{\HessEmp{\lambda}{-r}{\optimum} \Hess{\lambda}{r}{\optimum}}
    \end{equation*}
    Then, use the Hermitian inequalities of \cref{eq:lem_hermitian_ineq_diff_rleq1} in \cref{lem:hermitian_inequalities}, then use the concentration inequalities of \cref{prop:main_concentration_bound_operators}. Both can be found in \cref{sec:technical_lemmas_appendix}. In details:
    \begin{itemize}
        \item If $r \leq 1/2$, use  then the concentration inequality of \cref{eq:prop_concentration_sqrt}:
        \begin{align*}
            \norm{\HessEmp{\lambda}{-r}{\optimum} \Hess{\lambda}{r}{\optimum}} 
            &\leq \norm{\HessEmp{\lambda}{-1/2}{\optimum} \Hess{\lambda}{1/2}{\optimum}}^{2r} \\
            &\leq 2^{r/2} \iftext \Hyp{1}.
        \end{align*}
        with confidence $1 - \delta$.
        \item If $r > 1/2$, use the concentration inequality of \cref{eq:prop_concentration_inv}:
        \begin{align*}
            \norm{\HessEmp{\lambda}{-r}{\optimum} \Hess{\lambda}{r}{\optimum}} 
            &\leq \norm{\HessEmp{\lambda}{-1}{\optimum} \Hess{\lambda}{}{\optimum}}^{r} \\
            &\leq 2^{r} \iftext \Hyp{1b}: n \geq 8 \frac{\cstBTwoStr^2}{\lambda^2} \log^2 \frac{2}{\delta}.
        \end{align*}
    \end{itemize}
    All in all, after simplification, the bound on the operator norm when $r \leq 1$ reads 
    \begin{equation}\label{eq:bound_bias_glue_3}
        \norm{\lambda^t \HessEmp{\lambda}{-(t-1/2)}{\optimum}\Hess{}{r}{\optimum}}
        \leq \lambda^s (1 \lor (\cstBTwoStr + \lambda)) 2^r \quad \iftext \begin{cases}
            \Hyp{1}  &\text{ when } r \leq 1/2 \\
            \Hyp{1b} &\text{ when } r > 1/2,
        \end{cases}
    \end{equation}
    with confidence $1-\delta$. 
    We now turn to the case $r > 1$. 

    ~\paragraph{Case $r > 1$.}
    We tackle this case with a different decomposition:
    \begin{align*}
        \norm{\lambda^t \HessEmp{\lambda}{-(t-1/2)}{\optimum}\Hess{}{r}{\optimum}} 
        &\leq \norm{\lambda^t \HessEmp{\lambda}{-(t-1/2)}{\optimum}\HessEmp{}{r}{\optimum}} + \norm{\lambda^t \HessEmp{\lambda}{-(t-1/2)}{\optimum} (\Hess{}{r}{\optimum} - \HessEmp{}{r}{\optimum})}
    \end{align*}
    Looking at the first term; recalling that $\HessEmp{}{}{\optimum} \leq \cstBTwoStr$, we have:
    \begin{align*}
        \norm{\lambda^t \HessEmp{\lambda}{-(t-1/2)}{\optimum}\HessEmp{}{r}{\optimum}} 
        &\leq \sqrt{\lambda} \sup_{0 < \sigma \leq \cstBTwoStr} \p{\frac{\lambda}{\lambda + \sigma}}^{t-1/2} \sigma^r \\
        &\leq \lambda^s \begin{cases}
            \frac{r}{(t-1/2)^r} \quad &\iftext \; r + 1/2 < t \\
            \frac{\cstBTwoStr^r}{\p{\cstBTwoStr^r + \lambda}^{t-1/2}} \quad &\otwtext
        \end{cases} \\
        &\leq \lambda^s \begin{cases}
            \frac{r}{(t-1/2)^r} \quad &\iftext \; r + 1/2 < t \\
            \cstBTwoStr^{r-t+1/2} \quad &\otwtext
        \end{cases}
    \end{align*}
    where we used the computation of \cref{lem:bound_residual_it_spectral_function}. The second term can be upper bounded as follows:
    \begin{align*}
        \norm{\lambda^t \HessEmp{\lambda}{-(t-1/2)}{\optimum} (\Hess{}{r}{\optimum} - \HessEmp{}{r}{\optimum})}
        &\leq \norm{\lambda^t \HessEmp{\lambda}{-(t-1/2)}{\optimum}} \norm{\Hess{}{r}{\optimum} - \HessEmp{}{r}{\optimum}} \\
        &\leq w(r) \sqrt{\lambda} \norm{\Hess{}{}{\optimum} - \HessEmp{}{}{\optimum}} \\
        &\leq w(r) \frac{\lambda^s}{(t-1/2)^r} \iftext \Hyp{2} : n \geq 2 \p{1 \lor \p{\frac{2 \cstBTwoStr (t-1/2)^r}{\lambda^{s-1/2}}}^2} \log \frac{2}{\delta}
    \end{align*}
    with confidence $1 - \delta$. We applied \cref{eq:lem_hermitian_ineq_diff_rg1} in \cref{lem:hermitian_inequalities} on the second inequality, and \cref{eq:prop_concentration_HS} in \cref{prop:main_concentration_bound_operators} for the last inequality, both of which can be found in \cref{sec:technical_lemmas_appendix}. We used:
    \begin{equation}\label{eq:definition_w}
        w(r) = r 2^{\intpart{r}+1} \cstBTwoStr^r.
    \end{equation}
    Thus, the bound on the operator norm when $r>1$ reads: 
    \begin{equation}\label{eq:bound_bias_glue_4}
        \norm{\lambda^t \HessEmp{\lambda}{-(t-1/2)}{\optimum}\Hess{}{r}{\optimum}} 
        \leq \lambda^s \begin{cases}
            \frac{w(r) + r}{(t-1/2)^r} \quad &\iftext \; r + 1/2 < t \\
            \frac{w(r)}{(t-1/2)^r} + \cstBTwoStr^{r-t+1/2} \quad &\otwtext
        \end{cases} \quad \iftext \Hyp{2},
    \end{equation}
    with confidence $1-\delta$. 

    ~\paragraph{Gluing things together.}
    We proceed to the conclusion. Define the following conditions:
    \begin{align*}
        \Hyp{1}: n &\geq 24 \frac{\cstBTwoStr}{\lambda} \log\frac{16 \cstBTwoStr}{\lambda\delta}, \\
        \Hyp{1b}: n &\geq 8 \frac{\cstBTwoStr^2}{\lambda^2} \log^2 \frac{4}{\delta}, \\
        \Hyp{2}: n &\geq 2 \csb{1 \lor \p{\frac{2 \cstBTwoStr (t-1/2)^r}{\lambda^{s-1/2}}}^2} \log \frac{4}{\delta},
    \end{align*}
    where we replace $\delta$ by $\delta / 2$ in order to have bounds with confidence $1-\delta/2$, so that the overall bound holds with confidence $1 - \delta$ (in fact, $1 - \delta/2$ in the first case).
    Now assume the following:
    \begin{align*}
        \Hyp{1} \quad &\iftext r \leq 1/2, \\
        \Hyp{1} + \Hyp{1b} \quad &\iftext 1/2 < r \leq 1, \\
        \Hyp{1} + \Hyp{2} \quad &\iftext 1 < r.
    \end{align*}
    Then, we can chain the inequalities of \cref{eq:bound_bias_glue_1,eq:bound_bias_glue_2,eq:bound_bias_glue_3,eq:bound_bias_glue_4}. We obtain:
    \begin{equation}
        \norm{\estEmp{\lambda}{t} - \estInt{\lambda}{t}}_{\Hess{}{}{\optimum}}
        \leq \sqrt{2} \norm{v} \cstP{t} \lambda^s \begin{cases}
            (1 \lor (\cstBTwoStr + \lambda)) 2^r            \quad &\iftext r \leq 1, \\
            \frac{w(r) + r}{(t-1/2)^r}                      \quad &\iftext r > 1 \andtext r + 1/2 < t, \\
            \frac{w(r)}{(t-1/2)^r} + \cstBTwoStr^{r-t+1/2}  \quad &\iftext r > 1 \andtext r + 1/2 \geq t,
        \end{cases}
    \end{equation}
    with confidence $1 - \delta$.
\end{proof}

\subsection{Bounding the variance}
After bounding the bias, we study the variance term: $\norm{\estInt{\lambda}{t} - \optimum}_{\Hess{}{}{\optimum}}$. 

\begin{theorem}[Optimal variance of Iterated Tikhonov estimator]\label{th:bound_on_variance}
    Let $\delta \in (0, 1]$. Recall the definition of the degrees of freedom $\df_\lambda$. Define the following conditions on the number of samples:
    \begin{equation*}
        \begin{aligned}
            \Hyp{1}: n &\geq 24 \frac{\cstBTwoStr}{\lambda} \log\frac{16 \cstBTwoStr}{\lambda\delta}, \\
            \Hyp{3}: n &\geq 2 \frac{\cstBOneStr^2}{\lambda \df_\lambda} \log \frac{4}{\delta}.
        \end{aligned}
    \end{equation*}
    Then, with probability greater than $1-\delta$:
    \begin{equation*}
        \norm{\estInt{\lambda}{t} - \optimum}_{\Hess{}{}{\optimum}} \leq 4 \sqrt{2} t \cstR{t} \sqrt{\frac{\df_\lambda}{n}} \cdot \sqrt{\log 2/\delta},
    \end{equation*}
    where we introduced:
    \begin{equation*}
        \cstR{t} \eqdef \prod_{k=1}^t \phibot^{-1}\p{\tOpEmp(\estInt{\lambda}{k} - \optimum)}.
    \end{equation*}
\end{theorem}
\begin{proof}
    The proof begins similarly to the study of the bias term (\cref{th:bound_on_bias}).
    ~\paragraph{Changing the norm.}
    We have the following bound (proof of \cref{th:bound_on_bias}, \cref{eq:bound_bias_glue_1}):
    \begin{equation}\label{eq:bound_variance_glue_1}
        \Hyp{1}: n \geq 24 \frac{\cstBTwoStr}{\lambda} \log\frac{8 \cstBTwoStr}{\lambda\delta} \implies \norm{\estInt{\lambda}{t} - \optimum}_{\Hess{}{}{\optimum}} \leq \sqrt{2} \norm{\estInt{\lambda}{t} - \optimum}_{\HessEmp{\lambda}{}{\optimum}}.
    \end{equation}

    ~\paragraph{Upper bounds on gradient.}
    To ease the notation, we denote by $\csta{k} = \phibot^{-1}\p{\tOpEmp(\estInt{\lambda}{k} - \optimum)}$. We have, thanks to the lower bound on gradient of \cref{eq:prop_gsc_gradient_lowerbound}:
    \begin{align*}
        \norm{\estInt{\lambda}{t} - \optimum}_{\HessEmp{\lambda}{}{\optimum}}
        &\leq \csta{t} \norm{\nabla \lossLEmp{\lambda}{}{\estInt{\lambda}{t}} - \nabla \lossLEmp{\lambda}{}{\optimum}}_{\HessEmp{\lambda}{-1}{\optimum}} \\
        &= \csta{t} \norm{\lambda(\estInt{\lambda}{t-1} - \optimum) - \nabla\lossLEmp{}{}{\optimum}}_{\HessEmp{\lambda}{-1}{\optimum}} \\
        &= \csta{t} \norm{\lambda\HessEmp{\lambda}{-1}{\optimum}(\estInt{\lambda}{t-1} - \optimum)}_{\HessEmp{\lambda}{}{\optimum}} + \csta{t} \norm{\nabla\lossLEmp{}{}{\optimum}}_{\HessEmp{\lambda}{-1}{\optimum}} \\
        &\leq \csta{t} \csta{t-1} \norm{\lambda^2 \HessEmp{\lambda}{-2}{\optimum}(\estInt{\lambda}{t-2} - \optimum)}_{\HessEmp{\lambda}{}{\optimum}} + \csb{\csta{t} + \csta{t-1} \norm{\lambda \HessEmp{\lambda}{-1}{\optimum}}}\norm{\nabla\lossLEmp{}{}{\optimum}}_{\HessEmp{\lambda}{-1}{\optimum}}
    \end{align*}
    We can unfold the recursion. The first term will disappears thanks to $\estInt{\lambda}{0} = \optimum$, and we are left with:
    \begin{align}
        \norm{\estInt{\lambda}{t} - \optimum}_{\HessEmp{\lambda}{}{\optimum}}
        &\leq \sum_{k=0}^{t-1} \p{\prod_{i=t-k}^{t} \csta{i}} \norm{\lambda^k \HessEmp{\lambda}{-k}{\optimum}} \norm{\nabla\lossLEmp{}{}{\optimum}}_{\HessEmp{\lambda}{-1}{\optimum}} \\
        &\leq \cstR{t} \norm{\HessEmp{\lambda}{-1/2}{\optimum} \Hess{\lambda}{1/2}{\optimum}} \p{\sum_{k=0}^{t-1} \norm{\lambda^k \HessEmp{\lambda}{-k}{\optimum}}} \norm{\nabla\lossLEmp{}{}{\optimum}}_{\Hess{\lambda}{-1}{\optimum}}, \\
        \text{where} \quad \cstR{t} &\eqdef \prod_{k=1}^t \phibot^{-1}\p{\tOpEmp(\estInt{\lambda}{k} - \optimum)}.
    \end{align}
    Consider the prefactor of $\normtxt{\nabla\lossLEmp{}{}{\optimum}}_{\Hess{\lambda}{-1}{\optimum}}$. We will bound $\normtxt{\HessEmp{\lambda}{-1/2}{\optimum} \Hess{\lambda}{1/2}{\optimum}}$ by $\sqrt{2}$ with the same concentration argument as for the bias. The sum is more difficult to deal with. By computing the supremum of $\sigma\mapsto \lambda^k / (\sigma + \lambda)^k$ we would find that the first $\intpart{t/2}$ terms have their maximum in $0$. We would end up with a bound for the sum of the order of $t/2$. We rather use the simpler, if not optimal, following bound:
    \begin{equation*}
        \sum_{k=0}^{t-1} \norm{\lambda^k \HessEmp{\lambda}{-k}{\optimum}} \leq t.
    \end{equation*}
    It is suboptimal, but of the same order of an exact computation of the operator norm. Thus, we now have:
    \begin{equation}\label{eq:bound_variance_glue_2}
        \norm{\estInt{\lambda}{t} - \optimum}_{\HessEmp{\lambda}{}{\optimum}} 
        \leq \sqrt{2} t \cstR{t} \norm{\nabla\lossLEmp{}{}{\optimum}}_{\Hess{\lambda}{-1}{\optimum}} \quad \text{when } \Hyp{1}
    \end{equation}

    ~\paragraph{Bounding the gradient $\normtxt{\nabla\lossLEmp{}{}{\optimum}}_{\Hess{\lambda}{-1}{\optimum}}$.}
    We use a plain Bernstein inequality to bound the gradient, as in \cref{prop:main_concentration_bound_operators}:
    \begin{equation*}
        \Hess{\lambda}{-1/2}{\optimum} \nabla \lossLEmp{}{}{\optimum} = \frac{1}{n} \sum_{k=1}^n \Hess{\lambda}{-1/2}{\optimum} \nabla \ell_{z_i}(\optimum).
    \end{equation*}
    We have
    \begin{align*}
        \sup_{z \in \Supp \rho} \norm{\nabla\ell_{z} (\optimum)}_{\Hess{\lambda}{-1}{\optimum}} &\leq \frac{\cstBOneStr}{\sqrt{\lambda}}, \\
        \andtext \quad \EE_{z \sim \rho}\csb{\norm{\nabla\ell_{z} (\optimum)}_{\Hess{\lambda}{-1}{\optimum}}}^2 &\eqdef \df_\lambda.
    \end{align*}
    With confidence $1 - \delta$, we now have
    \begin{equation*}
        \norm{\nabla\lossLEmp{}{}{\optimum}}_{\Hess{\lambda}{-1}{\optimum}} 
        \leq \frac{\cstBOneStr}{\sqrt{\lambda}} \frac{2 \log 2/\delta}{n} + \sqrt{\df_\lambda \frac{2 \log 2/\delta}{n}}.
    \end{equation*}
    We simplify this equation. Assuming
    \begin{align*}
        \Hyp{3}: n \geq 2 \frac{\cstBOneStr^2}{\lambda \df_\lambda} \log \frac{2}{\delta}
    \end{align*}
    we get the bound:
    \begin{equation}\label{eq:bound_variance_glue_3}
        \norm{\nabla\lossLEmp{}{}{\optimum}}_{\Hess{\lambda}{-1}{\optimum}} 
        \leq 2 \sqrt{2} \sqrt{\df_\lambda \frac{2 \log 2/\delta}{n}}.
    \end{equation}

    ~\paragraph{Gluing things together.}
    All in all, we can glue together the inequalities in \cref{eq:bound_variance_glue_1,eq:bound_variance_glue_2,eq:bound_variance_glue_3}. We obtain:
    \begin{equation*}
        \norm{\estInt{\lambda}{t} - \optimum}_{\Hess{}{}{\optimum}} \leq 4 \sqrt{2} t \cstR{t} \sqrt{\frac{\df_\lambda}{n}} \cdot \sqrt{\log 2/\delta} \quad \text{when } \Hyp{1} + \Hyp{3}
    \end{equation*}
    with confidence $1 - 2\delta$. We obtain the statement of the theorem by replacing $\delta$ with $\delta/2$, so that the result holds with confidence $1 - \delta$. 
\end{proof}

\subsection{Conditions for non-exponentials prefactors}

The prefactors $\cstP{t}$ and $\cstR{t}$ are hard to bound; they can depend exponentially on $\norm{\optimum}$ in the worst case \cite{regularized-erm-gsc}. The purpose of this section is to give sufficient conditions on the number of samples $n$ for those quantities to turn constant. The key quantity to compare to is the \textit{Dikin radius} \cite{regularized-erm-gsc,NEURIPS2019_60495b4e}.
\begin{definition}[Dikin radius]\label{def:dikin_radius}
    For $\theta \in \HH$ and $\lambda > 0$, define $\rOp{\lambda}{\theta}$ s.t
    \begin{equation}\label{eq:def_Dikin_radius}
        \frac{1}{\rOp{\lambda}{\theta}} = \sup_{z\in\Supp\rho} \sup_{g \in \phi(z)} \norm{g}_{\Hess{\lambda}{-1}{\theta}}.
    \end{equation}
\end{definition}
The inverse of the Dikin radius can be upper bounded by $R/\sqrt{\lambda}$. However, we prefer keeping bounds in $\rStar$. Indeed, they take into account the geometry of the loss function around the optimum, and are thus much more precise. 

Note that in the following, we might be content with the \textit{empirical} Dikin radius $\widehat{\rOp{\lambda}{\theta}}$, ie. replacing $\rho$ by $\hat{\rho}$ in the previous definition. So as not to ladden the notations and have something independant of the sampling, we use the fact that $\Supp \hat{\rho} \subset \Supp \rho$ to ensure that:
\begin{equation*}
    \frac{1}{\widehat{\rOp{\lambda}{\theta}}} \leq \frac{1}{\rOp{\lambda}{\theta}} \quad \andtext \quad \tOpEmp(\cdot) \leq \tOp(\cdot).
\end{equation*}
Finally, we will use the following notation:
\begin{equation}
    \rStar \eqdef \rOp{\lambda}{\optimum}.
\end{equation}

\subsubsection{Pefactor of the variance}
We first proceed with the prefactor of the variance $\cstR{t}$.

\begin{proposition}[Constant prefactor for the variance]\label{prop:constant_prefactor_variance}
    The following condition:
    \begin{equation}
        \Hyp{4}: n \geq 8 (et)^2 \p{4 \lor C^2t^2} \frac{\df_\lambda}{\rStar} \log 2/\delta,
    \end{equation}
    where $C \leq 0.8$ is a constant, is sufficient to guarantee that
    \begin{equation}
        \cstR{t} \eqdef \prod_{k=1}^t \phibot^{-1}\p{\tOpEmp(\estInt{\lambda}{k} - \optimum)} \leq e.
    \end{equation}
\end{proposition}
\begin{proof}
    ~\paragraph{A first bound.}
    Note that:
    \begin{align*}
        \tOp(\estInt{\lambda}{t} - \optimum) 
        &= \sup_{z \in \Supp \rho} \sup_{g \in \phi(z)} \absv{g \cdot (\estInt{\lambda}{t} - \optimum)} \\
        &\leq \sup_{z \in \Supp \rho} \sup_{g \in \phi(z)} \norm{g}_{\HessEmp{\lambda}{-1}{\optimum}} \norm{\estInt{\lambda}{t} - \optimum}_{\HessEmp{\lambda}{}{\optimum}},
    \end{align*}
    which gives us a bound we will use multiple times:
    \begin{equation}\label{eq:localization_of_t}
        \tOp(\estInt{\lambda}{t} - \optimum) 
        \leq \frac{\norm{\estInt{\lambda}{t} - \optimum}_{\HessEmp{\lambda}{}{\optimum}}}{\rStar}.
    \end{equation}
    We simply used the definition of the Dikin radius in \cref{eq:def_Dikin_radius}.
    
    We now use an upper bound of the numerator, available in the proof of \cref{th:bound_on_variance}:
    \begin{align*}
        \tOp(\estInt{\lambda}{t} - \optimum) 
        &\leq \cstR{t} \csb{2 \sqrt{2} t \sqrt{\log 2/\delta}} \sqrt{\frac{\df_\lambda}{n \rStar}} \\
        \iff \quad \tOp(\estInt{\lambda}{t} - \optimum) \phibot\p{\tOp(\estInt{\lambda}{t} - \optimum)} 
        &\leq \cstR{t-1} \csb{2 \sqrt{2} t \sqrt{\log 2/\delta}} \sqrt{\frac{\df_\lambda}{n \rStar}} \eqdef X_{t-1}.
    \end{align*}
    Now, using the fact that $x \phibot(x) = 1 - e^{-x}$, we get that
    \begin{equation}\label{eq:basic_localization_trick}
    \begin{aligned}
        \tOp(\estInt{\lambda}{t} - \optimum) &\leq - \log (1 - X_{t-1}) \\
        \csta{t} \eqdef \phibot^{-1}\p{\tOp(\estInt{\lambda}{t} - \optimum)} &\leq -X_{t-1}^{-1} \log (1-X_{t-1}) \eqdef h(X_{t-1}).
    \end{aligned}
    \end{equation}

    ~\paragraph{Recursion hypotheses.}
    The idea is to ensure:
    \begin{enumerate}
        \item $X_{k-1} \leq 1/2$  so that
        \begin{equation*}
            h(X_{k-1}) \leq 1 + C X_{k-1}
        \end{equation*}
        with $C$ a numeric constant s.t $h(1/2) = 1 + C/2$, which implies that $C \leq 0.8$. We are simply upper bounding $h$ which is convex on $\csb{0, 1/2}$. 
        \item $\csta{k} \leq 1 + 1/t$ for all $k \leq t$, so that we can have:
        \begin{align*}
            \cstR{t} &= \prod_{k=1}^t \csta{k} = \exp \sum_{k=1}^t \log (\csta{k}) \\
            &\leq \exp \sum_{k=1}^t \log (1 + 1/t) \leq e. \\
        \end{align*}
    \end{enumerate}

    ~\paragraph{Recursion.}
    Set $k=1$. Then $\cstR{0} = 1$ and to have 
    \begin{equation*}
        X_0 \leq 1/2 \quad \text{that is} \quad \csb{2 \sqrt{2} t \sqrt{\log 2/\delta}} \sqrt{\frac{\df_\lambda}{n \rStar}} \leq \frac{1}{2},
    \end{equation*}
    it is sufficient to have
    \begin{equation*}
        n \geq N_0 \eqdef 32 t^2 \frac{\df_\lambda}{\rStar} \log 2/\delta.
    \end{equation*}
    We want to enforce
    \begin{equation*}
        \csta{1} \leq 1 + 1/t. 
    \end{equation*}
    A sufficient condition is
    \begin{align*}
        h(X_0) \leq 1 + C X_0 \leq 1 + 1/t
        &\impliedby X_0 \leq 1/tC \\
        &\impliedby n \geq N_0' \eqdef 8 t^4 C^2 \frac{\df_\lambda}{\rStar} \log 2/\delta.
    \end{align*}
    Now, let $k < n$. Assume the two conditions hold at step $k-1$. Then, $\cstR{k-1} \leq e$ and
    \begin{equation*}
        n \geq N_{k-1} \eqdef 32 (et)^2 \frac{\df_\lambda}{\rStar} \log 2/\delta \quad \text{implies} \quad X_{k-1} \leq \frac{1}{2}.
    \end{equation*}
    Likewise, 
    \begin{equation*}
        n \geq N_{k-1}' \eqdef 8 t^4 (Ce)^2 \frac{\df_\lambda}{\rStar} \log 2/\delta
    \end{equation*}
    gives 
    \begin{equation*}
        X_{k-1} \leq 1/tC, \quad \text{so that} \quad \csta{k} \leq 1 + 1/k.
    \end{equation*}

    ~\paragraph{Conclusion.}
    All in all, requiring
    \begin{equation*}
        \Hyp{4}: n \geq 8 (et)^2 \p{4 \lor C^2t^2} \frac{\df_\lambda}{\rStar} \log 2/\delta
    \end{equation*}
    is sufficient to have $\cstR{k} \leq e$, for any $k \leq t$. 
\end{proof}

\subsubsection{Pefactor of the bias}
The prefactor of the bias can be treated similarly. The only difficulty comes from the large number of subcases. Remember from \cref{th:bound_on_bias} that we have, with appropriate hypotheses,
\begin{equation}\label{eq:easing_bound_prefactor_bias}
    \norm{\estEmp{\lambda}{t} - \estInt{\lambda}{t}}_{\HessEmp{\lambda}{}{\optimum}} \leq \cstP{t} \cstT(r, t) \lambda^{s}, \quad \text{with } s= (r+1/2) \land t.
\end{equation}

\begin{proposition}[Constant prefactor for the bias]\label{prop:constant_prefactor_bias}
    Assume $\Hyp{4}$ and
    \begin{equation*}
        \Hyp{5}: \lambda \leq \cstLambd{} \eqdef \csb{e^{t+2} \cstT(r,t) (2 \land Ct)}^{-1/(r+1/2\land 1)}.
    \end{equation*}
    Then
    \begin{equation*}
        \cstP{t} \leq e^{t+2}.
    \end{equation*}
\end{proposition}
\begin{proof}
    The proof is almost identical to the proof of \cref{prop:constant_prefactor_variance}. Let us simply point out the differences. We will drop the dependance of $\cstT$ on $r, t$ in the notation for simplicity.
    ~\paragraph{A first bound.}
    Here, we have that:
    \begin{equation}\label{eq:int_prefactor_bias}
        \tOp\p{\estEmp{\lambda}{k} - \estInt{\lambda}{k}} \leq \frac{\norm{\estEmp{\lambda}{k} - \estInt{\lambda}{k}}_{\HessEmp{\lambda}{}{\optimum}}}{\rStar} \leq \frac{\cstP{t}\cstT \lambda^s}{\rStar}, \quad \text{with } s = r+1/2 \land k.
    \end{equation}
    We used \cref{eq:easing_bound_prefactor_bias} in the second inequality. Recall the definition
    \begin{equation*}
        \cstP{t} \eqdef \prod_{k=1}^t \phibot^{-1}\p{\tOpEmp(\estEmp{\lambda}{k} - \estInt{\lambda}{k})} e^{\tOpEmp(\estInt{\lambda}{k} - \optimum)}.
    \end{equation*}
    Thanks to $\Hyp{4}$, we have $\cstR{t} \leq e$. Specifically, noting that $\phibot^{-1}(x) \geq x$, we have from the proof of \cref{prop:constant_prefactor_variance}: 
    \begin{equation*}
        1 + 1/t \geq \phibot^{-1}\p{\tOp(\estInt{\lambda}{k} - \optimum)} \geq \tOp(\estInt{\lambda}{k} - \optimum) \implies \prod_{k=1}^t e^{\tOp(\estInt{\lambda}{k} - \optimum)} \leq e^{t+1}.
    \end{equation*}
    Thus, we have that
    \begin{equation*}
        \cstP{k} \leq \underbrace{\prod_{i=1}^k \phibot^{-1}\p{\tOp(\estEmp{\lambda}{i} - \estInt{\lambda}{i})}}_{\cstQ{k}} e^{t+1}.
    \end{equation*}

    Dividing both sides in the \cref{eq:int_prefactor_bias} with $\phibot^{-1}\p{\tOp(\estEmp{\lambda}{k} - \estInt{\lambda}{k})}$, we obtain that
    \begin{equation*}
        \tOp\p{\estEmp{\lambda}{k} - \estInt{\lambda}{k}}\phibot\p{\tOp(\estEmp{\lambda}{k} - \estInt{\lambda}{k})} \leq \frac{\cstQ{k-1} \cstT e^{t+1} \lambda^s}{\rStar},
    \end{equation*}
    and we can apply the same reasoning as for the variance. Using $t \phibot(t) = 1 - e^{-t}$, we have
    \begin{equation*}
        \begin{aligned}
            \csta{k} \eqdef \phibot^{-1}\p{\tOp(\estEmp{\lambda}{k} - \estInt{\lambda}{k})} &\leq -X_{k-1}^{-1} \log (1 - X_{k-1}) \\
            X_{k-1} &\eqdef \cstQ{k-1} \cstT e^{t+1} \lambda^s /\rStar.
        \end{aligned}
    \end{equation*}

    ~\paragraph{Recursion.}
    We then do the exact same reasoning to the variance, that is require at each step $X_{k-1} \leq 1/2$ and $\csta{k} \leq 1 + 1/t$. Here, this amounts to require
    \begin{equation*}
        \lambda \leq \cstLambd{s} \eqdef \csb{e^{t+2} \cstT (2 \land Ct)}^{-1/s}.
    \end{equation*}
    The $\cstLambd{s}$ is increasing with $s$. So
    \begin{equation*}
        \forall k \leq t, \quad \cstLambd{s} \leq \cstLambd{r+1/2 \land 1} \eqdef \cstLambd{}, \quad \text{with } s= r+1/2 \land k.
    \end{equation*}

    ~\paragraph{Conclusion.}
    Requiring
    \begin{equation*}
        \Hyp{5}: \lambda \leq \cstLambd{} \eqdef \csb{e^{t+2} \cstT (2 \land Ct)}^{-1/(r+1/2\land 1)}
    \end{equation*}
    is sufficient to ensure $\cstQ{t} \leq e$, so that $\cstP{t} \leq e^{t+2}$. 
\end{proof}

\subsection{Optimal rates for IT estimator}\label{sec:optimal_rates_it}
The bound on the bias and the variance holds if the number of samples is ``high enough''. The purpose of next proposition is to merge all these hypotheses together. Precisely, the hypotheses requires in each regime are summed up in \cref{tab:necessary_hypotheses}. 

\begin{table}
  \centering
  \caption{Hypotheses needed to bound the bias and the variance, depending on the source condition parameter $r$.}
  \label{tab:necessary_hypotheses}
  \begin{tabular}{cccc}
    \toprule
    \textbf{Source condition} & \textbf{Bias} & \textbf{Variance} & \textbf{Numerical prefactors} \\
    \midrule
    $0 < r \leq 1/2$ & $\Hyp{1}$     & \multirow{3}{*}{$\Hyp{1} + \Hyp{3}$} & \multirow{3}{*}{$\Hyp{4} + \Hyp{5}$} \\
    $1/2 < r < 1$    & $\Hyp{1} + \Hyp{1b}$ &                               &  \\
    $r \geq 1$       & $\Hyp{1} + \Hyp{2}$  &                               &  \\
    \bottomrule
  \end{tabular}
\end{table}

    \begin{proposition}[Satisfying the hypotheses $\Hyp{1-5}$ with bounds on $n$ and $\lambda$]\label{prop:hypotheses_and_samples}
    The following relations hold:
    \begin{align*}
        n \geq \cstN{0} \eqdef \frac{2}{\lambda} \csb{12 \cstBTwoStr \lor \frac{\cstBOneStr^2}{\df_\lambda}} \log \frac{4}{\delta} \csb{1 \lor \frac{4 \cstBTwoStr}{\lambda}} &\implies \Hyp{1} + \Hyp{3}, \\
        n \geq \cstN{1/2} \eqdef \frac{2}{\lambda} \csb{12 \cstBTwoStr \lor \frac{\cstBOneStr^2}{\df_\lambda} \lor \frac{4 \cstBTwoStr}{\lambda}} \log^2 \frac{4}{\delta} \csb{1 \lor \frac{4 \cstBTwoStr}{\lambda}} &\implies \Hyp{1} + \Hyp{1b} + \Hyp{3}, \\
        n \geq \cstN{1} \eqdef \frac{2}{\lambda} \csb{12 \cstBTwoStr \lor \frac{\cstBOneStr^2}{\df_\lambda} \lor \lambda \lor \p{\frac{2 \cstBTwoStr (t-1/2)^r}{\lambda^{r-1/2}}}^2}  \log \frac{4}{\delta} \csb{1 \lor \frac{4 \cstBTwoStr}{\lambda}} &\implies \Hyp{1} + \Hyp{2} + \Hyp{3}, \\
        n \geq \cstNBar \eqdef 8 (et)^2 \p{4 \lor C^2t^2} \frac{\df_\lambda}{\rStar} \log 2/\delta &\implies \Hyp{4}, \\
        \lambda \leq \cstLambd{} \eqdef \csb{e^{t+2} \cstT(r,t) (2 \land Ct)}^{-1/(r+1/2\land 1)} &\implies \Hyp{5}.
    \end{align*}
    Recall that $\cstT$ is defined in \cref{th:bound_on_bias}. Moreover, having
    \begin{equation*}
        \lambda = K n^{-\frac{\alpha}{1 + \alpha(2r + 1)}}
    \end{equation*}
    with $K$ a constant not depending on $n$ make all these conditions possible. 
\end{proposition}
\begin{proof}
    The expression of the constant boils down to taking the maximum of each expression. Recall that:
    \begin{itemize}
        \item $\Hyp{1}, \Hyp{1b}, \Hyp{2}$ are defined in \cref{th:bound_on_bias};
        \item $\Hyp{3}$ is defined in \cref{th:bound_on_variance};
        \item $\Hyp{4}$ is defined in \cref{prop:constant_prefactor_variance};
        \item $\Hyp{5}$ is defined in \cref{prop:constant_prefactor_bias}.
    \end{itemize}
    About the fact they are attainable, we need to check that the power of $n$ is smaller than $1$, in order that 
    \begin{equation*}
        \exists n, \quad n \geq \cstN{}(\lambda) \quad \andtext \quad \lambda = K n^{-\frac{\alpha}{1 + \alpha(2r + 1)}},
    \end{equation*}
    with  $\cstN{}(\lambda)$ chosen among $\cb{\cstN{0}, \cstN{1/2}, \cstN{1}, \cstNBar}$.
    In the following, $\sim$ denotes equality up to log factors between two quantities. Recall that $\cstSsmall \lambda^{-1/\alpha} \leq \df_\lambda \leq \cstS \lambda^{-1/\alpha}$, and assume
    \begin{equation*}
        \lambda = K n^{-\frac{\alpha}{1 + \alpha(2r + 1)}}
    \end{equation*}
    for some $K$ a positive constant.
    \begin{itemize}
        \item When $r \leq 1/2$, $\cstN{0} \sim \lambda^{-(1+1/\alpha)} \sim n^{\frac{1 + \alpha}{1 + \alpha(2r + 1)}}$ and $\frac{1 + \alpha}{1 + \alpha(2r + 1)} < 1$. 
        \item When $1/2 < r \leq 1$, $\cstN{1/2} \sim \lambda^{-2} \sim n^{\frac{2\alpha}{1 + \alpha(2r + 1)}}$ and $\frac{2\alpha}{1 + \alpha(2r + 1)} < 1$ as $\alpha(2r+1) > 2 \alpha$. 
        \item Finally, when $r > 1$, $\cstN{1} \sim \lambda^{-2r} \sim n^{\frac{\alpha (2r)}{1 + \alpha(2r + 1)}}$ and $\alpha (2r) < 1 + \alpha(2r + 1)$. 
        \item For $\cstNBar$, use the upper bound $\frac{\df_\lambda}{\rStar} \leq \cstS R \lambda^{-(1/\alpha + 1/2)}$. Then, $\cstNBar \sim n^{\frac{\alpha + 2}{2(1 + \alpha(2r+1))}}$ and $\frac{\alpha + 2}{2(1 + \alpha(2r+1))} = 1 - \frac{\alpha(4r+1)}{\alpha(4r+2) + 2} \leq 1$.
    \end{itemize}
\end{proof}

Having bounded the bias and the variance of the estimator, we are now in shape to state our main result. 
\begin{theorem}[Optimal rates of IT estimator]\label{th:optimal_rates}
    Let $\delta \in (0, 1]$, $\lambda > 0$ and choose $n$ so that $\Hyp{3}$ and the following holds:
    \begin{align*}
        \Hyp{1} \quad &\iftext r \leq 1/2, \\
        \Hyp{1} + \Hyp{1b} \quad &\iftext 1/2 < r \leq 1, \\
        \Hyp{1} + \Hyp{2} \quad &\iftext r > 1.
    \end{align*}
    Then we can bound the excess risk with probability greater than $1-\delta$ as
    \begin{equation*}
        \lossL{}{}{\estEmp{\lambda}{t}} - \lossL{}{}{\optimum} \leq \cstCBias \lambda^{2s}
        + \cstCVar \frac{\df_\lambda}{n}, \quad \text{with } s = (r+1/2) \land t.
    \end{equation*}
    If we further assume that the capacity condition holds and that the estimator does not saturate, that is $t \geq r+1/2$, then setting
    \begin{equation*}
        \lambda = \csb{\p{\frac{\cstCVar}{\cstCBias}}^2 \cstS}^{\frac{\alpha}{1+\alpha(2r+1)}} n^{-\frac{\alpha}{1+\alpha(2r+1)}}
    \end{equation*}
    makes the following holds with confidence $2\delta$:
    \begin{equation*}
        \lossL{}{}{\estEmp{\lambda}{t}} - \lossL{}{}{\optimum} \leq 2 \csb{\p{\frac{\cstCVar}{\cstCBias}}^2 \cstS}^{\frac{\alpha(2r +1)}{1+\alpha(2r+1)}} n^{-\frac{\alpha(2r +1)}{1+\alpha(2r+1)}},
    \end{equation*}
    where the constants $\cstCBias, \cstCVar$ are bounded by quantities only depending on $r,t, \cstBTwoStr, \delta$ as soon as hypotheses $\Hyp{4}$ and $\Hyp{5}$ are satisfied. 
\end{theorem}
\begin{proof}
    ~\paragraph{Decomposition of the risk.}
    We use the decomposition of the risk:
    \begin{align*}
        \lossL{}{}{\estEmp{\lambda}{t}} - \lossL{}{}{\optimum} 
        &\leq \Psi\p{\tOp(\estEmp{\lambda}{t} - \optimum)} \norm{\estEmp{\lambda}{t} - \optimum}_{\Hess{}{}{\optimum}}^2 \\
        &\leq 2 \Psi\p{\tOp(\estEmp{\lambda}{t} - \estInt{\lambda}{t}) + \tOp(\estInt{\lambda}{t} - \optimum)} \csb{\norm{\estEmp{\lambda}{t} - \estInt{\lambda}{t}}_{\Hess{}{}{\optimum}}^2 + \norm{\estInt{\lambda}{t} - \optimum}_{\Hess{}{}{\optimum}}^2},
    \end{align*}
    where we applied \cref{prop:properties_gsc}, and used that $(a+b)^2 \leq 2(a^2 + b^2)$. 

    ~\paragraph{Bias and variance prefactors.}
    We introduce the following quantities:
    \begin{align*}
        \cstCBias^2 &= 2\Psi\p{\tOp(\estEmp{\lambda}{t} - \estInt{\lambda}{t}) + \tOp(\estInt{\lambda}{t} - \optimum)} \cstT(r,t) \cstP{t}, \\
        \cstCVar^2 &= 2\Psi\p{\tOp(\estEmp{\lambda}{t} - \estInt{\lambda}{t}) + \tOp(\estInt{\lambda}{t} - \optimum)} \csb{4\sqrt{2}t \cstR{t} \sqrt{\log 2/\delta}},
    \end{align*}
    where $\cstP{t}, \cstR{t}$ are defined in \cref{th:bound_on_bias,th:bound_on_variance} respectively. Then the bound on the excess risk reads:
    \begin{equation*}
        \lossL{}{}{\estEmp{\lambda}{t}} - \lossL{}{}{\optimum} \leq \cstCBias \begin{cases}
            \lambda^{2r+1} \quad &\iftext r+1/2 \leq t \\
            \lambda^{2t} \quad &\otwtext
        \end{cases}
        + \cstCVar \frac{\df_\lambda}{n}
    \end{equation*}
    with confidence $2\delta$ with the appropriate hypothesis $\Hyp{1}, \Hyp{2}$ or $\Hyp{3}$, depending on $r$, see \cref{tab:necessary_hypotheses}.

    ~\paragraph{Optimal $\lambda$.}
    Further assume $t \geq r+1/2$ and the capacity condition holds with parameters $\cstS, \alpha$. Then, setting:
    \begin{equation*}
        \lambda^{\frac{1+\alpha(2r+1)}{\alpha}} = \p{\frac{\cstCVar}{\cstCBias}}^2 \frac{\cstS}{n} \iff \lambda = \csb{\p{\frac{\cstCVar}{\cstCBias}}^2 \cstS}^{\frac{\alpha}{1+\alpha(2r+1)}} n^{-\frac{\alpha}{1+\alpha(2r+1)}},
    \end{equation*}
    makes the following bound holds with probability $1-2\delta$:
    \begin{equation*}
        \lossL{}{}{\estEmp{\lambda}{t}} - \lossL{}{}{\optimum} \leq 2 \csb{\p{\frac{\cstCVar}{\cstCBias}}^2 \cstS}^{\frac{\alpha(2r +1)}{1+\alpha(2r+1)}} n^{-\frac{\alpha(2r +1)}{1+\alpha(2r+1)}}.
    \end{equation*}
    
    ~\paragraph{Explicit prefactors.}
    Assume Hyp. $\Hyp{4}$ and $\Hyp{5}$ hold, and $\lambda \leq \cstBTwoStr$. Then the quantities $\cstCBias, \cstCVar$ only depend on $r, t$ up to the term $\Psi\p{\tOp(\estEmp{\lambda}{t} - \estInt{\lambda}{t}) + \tOp(\estInt{\lambda}{t} - \optimum)}$. Noting that:
    \begin{equation*}
        1 + 1/t \geq \phibot^{-1}(x) \geq x \quad \text{implies} \quad 1 + 1/t \geq x, 
    \end{equation*}
    and $\Psi$ is increasing we have
    \begin{equation*}
        \Psi\p{\tOp(\estEmp{\lambda}{t} - \estInt{\lambda}{t}) + \tOp(\estInt{\lambda}{t} - \optimum)} \leq \Psi(4) \leq 4.
    \end{equation*}
    In the end $\cstCBias, \cstCVar$ only depend on $r, t$ and the parameters of the problem:
    \begin{equation}\label{eq:explicit_constants}
        \begin{aligned}
            \cstCBias^2 &\leq 8 \cstT(r,t) e^{t+2} \\
            \cstCVar^2 &\leq 32 t e \sqrt{\log 2 / \delta},
        \end{aligned}
    \end{equation}
    where $\cstT(r,t)$ was introduced previously in \cref{th:bound_on_bias}:
    \begin{equation*}
        \cstT(r, t) = \begin{cases}
            \norm{v} (1 \lor (\cstBTwoStr + \lambda)) 2^r            \quad &\iftext r \leq 1, \\
            \norm{v} \frac{w(r) + r}{(t-1/2)^r}                      \quad &\iftext r > 1 \andtext r + 1/2 < t, \\
            \norm{v} \frac{w(r)}{(t-1/2)^r} + \cstBTwoStr^{r-t+1/2}  \quad &\iftext r > 1 \andtext r + 1/2 \geq t.
        \end{cases}
    \end{equation*}
    
    ~\paragraph{Proof of \cref{th:optimal_rates_main_body} in the paper.}
    We took the maximum on the lower bounds on the samples to simplify the result in the main body. Simply define: 
    \begin{align*}
        \cstN{} &= \cstNBar \lor \begin{cases}
            \cstN{0} \quad &\iftext r\leq 1/2 \\
            \cstN{1/2} \quad &\iftext 1/2 < r < 1 \\
            \cstN{1} \quad &\otwtext \\
        \end{cases} \\
        \andtext \quad \cstCRisk &= \csb{\p{\frac{\cstCVar}{\cstCBias}}^2 \cstS}^{\frac{\alpha}{1+\alpha(2r+1)}}.
    \end{align*}
    Again, we highlight that the observation made in \cref{prop:hypotheses_and_samples} is key to ensure that these constants are attainable, in the sense that they are not in contradiction with the optimal rate in $n$.
\end{proof}

\section{Statistical guarantees with inexact solvers}
This section is devoted to finding a rule on the tolerance enforced at each step of the proximal sequence. Given a tolerance $\epsilon$, we look for $\bar{\epsilon}_1, \dots, \bar{\epsilon}_n$, the tolerance to ensure at each proximal step. It leads to \cref{prop:error_propagation_proximal_sequence_main} in the main body of the article.

\paragraph{Important remark on the notation.} So as to simplify the notation, we drop the hat $\hat{}$ on the loss function. That is, we simply take a loss function $L$ assumed to be GSC. In practice, this function is of course the empirical loss $\lossLEmp{}{}{}$. We denote with a bar $\bar{}$ the quantity we compute at each step, and whose aim is to approximate the estimator of $\lossL{}{}{}$.

\paragraph{Tikhonov regularization.}
For a GSC function $L$, we define:
\begin{align*}
    \estTrue{\mu}{1} = \prox{L/\mu} (0) &= \arg \min_{\theta} \lossL{\mu}{}{\theta}, \quad && &\lossL{\mu}{}{\theta} &\eqdef L(\theta) + \frac{\mu}{2} \norm{\theta}^2 \\
    \estTrue{\mu}{k+1} = \prox{L/\mu}(\estTrue{\lambda}{k}) &= \arg \min_{\theta} \lossL{\mu}{\lambda, k}{\theta}, \quad && &\lossL{\mu}{\lambda, k}{\theta} &\eqdef L(\theta) + \frac{\mu}{2} \norm{\theta - \estTrue{\lambda}{k}}^2 \\
    \estApprox{\mu}{k+1} = \prox{L/\mu}(\estApprox{\lambda}{k}) &= \arg \min_{\theta}\lossLApprox{\mu}{\lambda, k}{\theta}, \quad && &\lossLApprox{\mu}{\lambda, k}{\theta}&\eqdef L(\theta) + \frac{\mu}{2} \norm{\theta - \estApprox{\lambda}{k}}
\end{align*}
so that we can refer easily to the function which has to be minimized when evaluating the proximal operator. 

\subsection{Definitions}\label{sec:definitions_sec_inexact_solvers}

We use the following notations for the \emph{Newton decrement}:
\begin{itemize}
    \item The theoretical quantity writes: 
    \begin{equation*}
        \NwtdecTrue{\mu}{\lambda, k}{\theta} = \norm{\nabla \lossL{\mu}{\lambda, k}{\theta}}_{\Hess{\lambda}{-1}{\theta}} = \norm{\nabla \lossL{}{}{\theta} + \mu (\theta - \estTrue{\lambda}{\lambda, k})}_{\Hess{\mu}{-1}{\theta}};
    \end{equation*}
    \item The normalized Newton decrement is defined with:
    \begin{equation*}
        \NwtdecTrueNorm{\lambda}{k-1}{\theta} = \frac{\NwtdecTrue{\lambda}{k-1}{\theta}}{\rOp{\lambda}{\theta}};
    \end{equation*}
    \item The quantity we compute is: 
    \begin{equation*}
        \NwtdecApprox{\mu}{\lambda, k}{\theta} = \norm{\nabla \lossL{\mu}{\lambda, k}{x}}_{\Hess{\lambda}{-1}{\theta}} = \norm{\nabla \lossL{}{}{\theta} + \mu (\theta - \estApprox{\lambda}{\lambda, k})}_{\Hess{\mu}{-1}{\theta}}.
    \end{equation*}
\end{itemize}

We also recall some definition and properties. $R$ is defined with
\begin{equation*}
    R = \sup_{z\in\Supp \rho} \sup_{g \in \phi(z)} \norm{g} \quad \text{so that} \quad \rOp{\lambda}{\theta} \geq R / \sqrt{\lambda}
\end{equation*}
and $\rOp{\lambda}{\theta}$ is given in \cref{def:dikin_radius}.
The Dikin ellipsoid, as in \cite{NEURIPS2019_60495b4e}, reads
\begin{equation*}
    \forall \cc \in \RR, \quad \DD{\lambda}{k-1}{\cc} = \cb{\theta \in \HH ; \NwtdecTrueNorm{\lambda}{k-1}{\theta} \leq \cc}.
\end{equation*}

We provide a short lemma to show how controlling the \textit{normalized} Newton decrement enables to control quantities depending on $\tOp$.
\begin{lemma}[Localization properties with the Newton decrement]\label{lem:localization_prop_with_nwt_decrement}
    Let $k \leq t$ and $\cc > 0$. Assume
    \begin{equation*}
        \estApprox{\lambda}{k} \in \DD{\lambda}{k-1}{\cc}, \quad \text{that is} \quad \NwtdecTrueNorm{\lambda}{k-1}{\estApprox{\lambda}{k}} \leq \cc.
    \end{equation*}
    Then, we have
    \begin{equation}
        \phibot^{-1}\p{\tOp(\estApprox{\lambda}{k} - \estTrue{\lambda}{k})} \leq -\frac{1}{\cc} \log(1-\cc) \eqdef \kappa_\cc.
    \end{equation}
\end{lemma}
\begin{proof}
    The proof combines inequalities we already used, replacing the normalized gradient with the Newton decrement. Recall \cref{eq:localization_of_t}, which states that
    \begin{equation}\label{eq:proof_localization_nwt_dec_1}
        \tOp(\estApprox{\lambda}{k} - \estTrue{\lambda}{k})
        \leq \frac{\norm{\estApprox{\lambda}{k} - \estTrue{\lambda}{k}}_{\HessEmp{\lambda}{}{\estApprox{\lambda}{k}}}}{\rOp{\lambda}{\estApprox{\lambda}{k}}}.
    \end{equation}
    Using the lower bound on gradient of \cref{lem:stacking_operator_gradient} gives
    \begin{equation*}
        \norm{\estApprox{\lambda}{k} - \estTrue{\lambda}{k}}_{\HessEmp{\lambda}{}{\estApprox{\lambda}{k}}} 
        \leq \phibot^{-1}(\tOp(\estApprox{\lambda}{k} - \estTrue{\lambda}{k})) \norm{\nabla \lossL{\lambda}{k-1}{\estApprox{\lambda}{k}}}_{\Hess{\lambda}{-1}{\estApprox{\lambda}{k}}},
    \end{equation*}
    and using the definition of the Newton decrement in the previous equation gives
    \begin{equation}\label{eq:lower_bound_gradient_newton_decrement}
        \norm{\estApprox{\lambda}{k} - \estTrue{\lambda}{k}}_{\HessEmp{\lambda}{}{\estApprox{\lambda}{k}}} 
        \leq \phibot^{-1}(\tOp(\estApprox{\lambda}{k} - \estTrue{\lambda}{k})) \NwtdecTrue{\lambda}{k-1}{\estApprox{\lambda}{k}}.
    \end{equation}
    Plugging \cref{eq:lower_bound_gradient_newton_decrement} in \cref{eq:proof_localization_nwt_dec_1} implies
    \begin{equation*}
        \phibot(\tOp(\estApprox{\lambda}{k} - \estTrue{\lambda}{k})) \tOp(\estApprox{\lambda}{k} - \estTrue{\lambda}{k}) \leq \frac{\NwtdecTrue{\lambda}{k-1}{\estApprox{\lambda}{k}}}{\rOp{\lambda}{\estApprox{\lambda}{k}}} \eqdef \NwtdecTrueNorm{\lambda}{k-1}{\estApprox{\lambda}{k}}.
    \end{equation*}
    Use the fact that $\phibot(x) x = 1 - e^{-x}$ combined with the definition of the normalized Newton decrement to simplify both sides of the previous equation. After simplification, we obtain
    \begin{equation*}
        \tOp(\estApprox{\lambda}{k} - \estTrue{\lambda}{k}) \leq -\log \p{1 - \NwtdecTrueNorm{\lambda}{k-1}{\estApprox{\lambda}{k}}}.
    \end{equation*}
    Apply $\phibot^{-1}$ on both side to have
    \begin{equation*}
        \phibot^{-1} \p{\tOp(\estApprox{\lambda}{k} - \estTrue{\lambda}{k})} \leq - \NwtdecTrueNorm{\lambda}{k-1}{\estApprox{\lambda}{k}}^{-1} \log \p{1 - \NwtdecTrueNorm{\lambda}{k-1}{\estApprox{\lambda}{k}}},
    \end{equation*}
    and the conclusion follows with the fact that this is an increasing function of the normalized Newton decrement, which is upper bounded by $\cc$.
\end{proof}
This lemma will be useful in the following derivation, and provide some intuition on GSC loss function.

\begin{remark}\label{rem:intuition_gsc}
    \textit{Intuition for GSC loss function.}
    The purpose of working with Generalized self-concordant loss functions is to be able to control the deviation of the function with their local quadratic approximation. For $\theta \in \HH$, \cref{lem:localization_prop_with_nwt_decrement} gives us that we can bound quantities depending on $\tOp$ in the inequalities of GSC loss functions of \cref{prop:properties_gsc}. When $\theta$ is deep into $\DD{\lambda}{k-1}{}$, then $\tOp \to 1/2$ and the bounds of \cref{prop:properties_gsc} are tight. On the contrary, when $\theta$ leaves this ellipsoid, the upper bound diverges exponentially to infinity while the lower bound goes exponentially to $0$, making the deviation from the quadratic approximation very loose. 

    To conclude, a GSC function with high $R$ has small Dikin ellipsoids, and is far from its quadratic approximation. On the contrary, a GSC function with low $R$ will be close to its quadratic approximation; the Dikin ellipsoid is large. The extreme case is obtained when $\ell$ is the square loss. Then, $\phi = \cb{0}$, so $R = 0$, and the Dikin ellipsoid spans the whole space for any $\theta \in \HH$. This implies {\it e.g.} that the lower and upper bound on the gradient matches, making the quadratic approximation tight.
\end{remark}

\subsection{Error propagation}
In this section, we give a sufficient condition for achieving an $\epsilon$ error on a sequence of proximal operators. Indeed, we aim at minimizing $\lossL{\lambda}{t-1}{}$, but we do not have access to this function; only to its approximation $\lossLApprox{\lambda}{t-1}{}$. Relating both is the purpose of the next result.

\begin{proposition}[Error propagation with proximal sequence]\label{prop:error_propagation_proximal_sequence}
    Let $\cc > 0$. Assume that you can solve each subproblem with precision $\bar{\epsilon}_k$ and that you have a guarantee on the exact normalized decrement:
    \begin{equation*}
        \forall k \in \cb{1, \dots, t}, \quad 
        \begin{cases}
            \NwtdecApprox{\lambda}{k-1}{\estApprox{\lambda}{k}} \leq \bar{\epsilon}_k \\
            \estApprox{\lambda}{k} \in \DD{\lambda}{k-1}{\cc} \iff \NwtdecTrueNorm{\lambda}{k-1}{\estApprox{\lambda}{k}} \leq \cc
        \end{cases}
    \end{equation*}
    Then requiring:
    \begin{equation*}
        \forall k \in \cb{1, \dots, t}, \quad \bar{\epsilon}_k = \epsilon \frac{\kappa_\cc^{k-t}}{t}
    \end{equation*}
    with $\kappa_\cc = -\nicefrac{1}{\cc} \log(1 - \cc)$ suffice to achieve an error $\epsilon$:
    \begin{equation*}
        \NwtdecTrue{\lambda}{t-1}{\estApprox{\lambda}{t}} \leq \epsilon.
    \end{equation*}
    We can replace the condition $\estApprox{\lambda}{k} \in \DD{\lambda}{k-1}{\cc}$ with $\epsilon \leq \cc \sqrt{\lambda}/R$.
\end{proposition}

\begin{proof}
    Let us track the error step by step. Denote by $\epsilon_k$ the Newton decrement of the exact function at each step: 
    \begin{equation*}
        \forall k, \quad \epsilon_k \eqdef \NwtdecTrue{\lambda}{k-1}{\estApprox{\lambda}{k}}.
    \end{equation*}
    Consider the following decomposition at step $k$:
    \begin{align*}
        \NwtdecTrue{\lambda}{k-1}{\estApprox{\lambda}{k}} 
        &= \norm{\nabla \lossL{}{}{\estApprox{\lambda}{k}} + \lambda\p{\estApprox{\lambda}{k} - \estTrue{\lambda}{k-1}}}_{\Hess{\lambda}{-1}{\estApprox{\lambda}{k}}} \\
        &\leq \norm{\nabla \lossL{}{}{\estApprox{\lambda}{k}} + \lambda\p{\estApprox{\lambda}{k} - \estApprox{\lambda}{k-1}}}_{\Hess{\lambda}{-1}{\estApprox{\lambda}{k}}} + \norm{\lambda \p{\estApprox{\lambda}{k-1} - \estTrue{\lambda}{k-1}}}_{\Hess{\lambda}{-1}{\estApprox{\lambda}{k}}} \\ 
        &\leq \NwtdecApprox{\lambda}{k-1}{\estApprox{\lambda}{k}} + \lambda \norm{\Hess{\lambda}{-1/2}{\estApprox{\lambda}{k}}\Hess{\lambda}{-1/2}{\estTrue{\lambda}{k}}}\norm{{\estApprox{\lambda}{k-1} - \estTrue{\lambda}{k-1}}}_{\Hess{\lambda}{}{\estApprox{\lambda}{k-1}}} \\
        &\leq \NwtdecApprox{\lambda}{k-1}{\estApprox{\lambda}{k}} + \phibot^{-1}\p{\tOp(\estApprox{\lambda}{k-1} - \estTrue{\lambda}{k-1})} \NwtdecTrue{\lambda}{k-2}{\estApprox{\lambda}{k-1}}
    \end{align*}
    In the last inequality we used that $\norm{\Hess{\lambda}{-1/2}{\estApprox{\lambda}{k}}\Hess{\lambda}{-1/2}{\estTrue{\lambda}{k}}} \leq 1/\lambda$ and the relation between the distance in Hessian's norm and the Newton decrement of \cref{eq:lower_bound_gradient_newton_decrement}. Introducing the notation with epsilon, the last line is by definition 
    \begin{equation}\label{eq:prop_error_propagation_epsilon_induction}
        \epsilon_k \leq \bar{\epsilon}_k + \phibot^{-1}\p{\tOp(\estApprox{\lambda}{k-1} - \estTrue{\lambda}{k-1})} \epsilon_{k-1}.
    \end{equation}
    The first term $\bar{\epsilon}_k$ is the error we can control at each step whereas $\epsilon_{k-1}$ is the error of interest which increases with $k$.
    Using the fact that $\estApprox{\lambda}{k-1} \in \DD{\lambda}{k-2}{\cc}$, we have 
    \begin{equation*}
        \phibot^{-1}\p{\tOp(\estApprox{\lambda}{k-1} - \estTrue{\lambda}{k-1})} \leq -\frac{1}{\cc} \log(1-\cc) \eqdef \kappa_\cc
    \end{equation*}
    thanks to \cref{lem:localization_prop_with_nwt_decrement}. Thus, \cref{eq:prop_error_propagation_epsilon_induction} becomes 
    \begin{equation}\label{eq:prop_error_propagation_epsilon_induction_kappac}
        \epsilon_k 
        \leq \bar{\epsilon}_k + \kappa_\cc \epsilon_{k-1}.
    \end{equation}
    This being valid for all $i \leq k$ and since $\NwtdecApprox{\lambda}{0}{\estApprox{\lambda}{1}} = \NwtdecTrue{\lambda}{0}{\estApprox{\lambda}{1}}$ we obtain that
    \begin{equation}\label{eq:prop_error_propagation_epsilon_kt}
        \epsilon_k \leq \sum_{i=1}^k \bar{\epsilon}_i \kappa_\cc^{k-i}.
    \end{equation}
    Now plug the assumption of the proposition, namely that each problem is solved with precision 
    \begin{equation*}
        \bar{\epsilon}_k = \epsilon \frac{\kappa_\cc^{k-t}}{t}
    \end{equation*}
    and use \cref{eq:prop_error_propagation_epsilon_kt} at step $t$ to obtain
    \begin{equation}\label{eq:prop_error_propagation}
        \epsilon_t \leq \sum_{i=1}^t \kappa_\cc^{i-t} \kappa_\cc^{t-i} \frac{\epsilon}{t} = \epsilon.
    \end{equation}
    
    \paragraph{Replacing $\estApprox{\lambda}{k} \in \DD{\lambda}{k-1}{\cc}$ with $\epsilon \leq \cc \sqrt{\lambda}/R$.} Let $k \geq 1$. Then, having 
    \begin{equation*}
        \estApprox{\lambda}{k} \in \DD{\lambda}{k-1}{\cc}
    \end{equation*}
    amounts by definition to have
    \begin{equation*}
        \NwtdecTrueNorm{\lambda}{k-1}{\estApprox{\lambda}{k}} \leq \cc,
    \end{equation*}
    which is also equivalent to
    \begin{equation*}
        \NwtdecTrue{\lambda}{k-1}{\estApprox{\lambda}{k}} \leq \cc \rOp{\lambda}{\estApprox{\lambda}{k}}.
    \end{equation*}
    We can use the crude lower bound $\rOp{\lambda}{\cdot} \geq \sqrt{\lambda}/R$. Thus, the following implication holds:
    \begin{equation}\label{eq:prop_error_propagation_necessary_cond}
        \epsilon_k \eqdef \NwtdecTrue{\lambda}{k-1}{\estApprox{\lambda}{k}} \leq \cc \frac{\sqrt{\lambda}}{R} \implies \estApprox{\lambda}{k} \in \DD{\lambda}{k-1}{\cc}.
    \end{equation}
    Now, assume $\epsilon \leq \cc \sqrt{\lambda}/R$. Then, we have that
    \begin{equation*}
        \epsilon_1 = \bar{\epsilon}_1 = \epsilon \frac{\kappa_\cc^{1-t}}{t} \leq \cc \sqrt{\lambda} / R \frac{\kappa_\cc^{1-t}}{t},
    \end{equation*}
    which gives
    \begin{equation*}
        \epsilon_1 \leq \cc \sqrt{\lambda}/R,
    \end{equation*}
    which implies $\estApprox{\lambda}{1} \in \DD{\lambda}{0}{\cc}$ following \cref{eq:prop_error_propagation_necessary_cond}. Then \cref{eq:prop_error_propagation_epsilon_induction_kappac} holds with $k=2$:
    \begin{equation*}
        \epsilon_2 \leq \bar{\epsilon}_2 + \kappa_\cc \epsilon_1.
    \end{equation*}
    For bigger $k$, proceed by induction. Let $k < t$ and assume for any $i < k$ that 
    \begin{equation*}
        \epsilon_{i+1} \leq \bar{\epsilon}_{i+1} + \kappa_\cc \epsilon_i.
    \end{equation*}
    Then, we have that 
    \begin{equation*}
        \epsilon_k \leq \sum_{i=1}^k \bar{\epsilon}_i \kappa_\cc^{k-i}
    \end{equation*}
    which gives the following bound, thanks to the assumption on $\epsilon$ and the $\bar{\epsilon}_i$:
    \begin{equation*}
        \epsilon_k \leq \sum_{i=1}^k \epsilon \frac{\kappa_\cc^{i-t}}{t} \kappa_c^{k-i} \leq \cc \sqrt{\lambda} / R.
    \end{equation*}
    This implies $\estApprox{\lambda}{k} \in \DD{\lambda}{k-1}{\cc}$ following \cref{eq:prop_error_propagation_necessary_cond}, and \cref{eq:prop_error_propagation_epsilon_induction_kappac} holds at step $k+1$. Thus the induction hypothesis holds for all $k$ and the conclusion of \cref{eq:prop_error_propagation} holds. 
    
    ~\paragraph{Proof of \cref{prop:error_propagation_proximal_sequence_main}.} This result is a direct application of the previous one, where we set $\cc = 1/2$. 
\end{proof}

We see that the requirement $\epsilon \leq \cc \sqrt{\lambda}/R$ is simply to ensures that a bound on the Newton decrement $\NwtdecTrue{\lambda}{k-1}{\estApprox{\lambda}{k}}$ translates to a bound on the \textit{normalized} Newton decrement $\NwtdecTrueNorm{\lambda}{k-1}{\estApprox{\lambda}{k}}$ \textit{via} the crude bound on the Dikin radius $\rOp{\lambda}{\estApprox{\lambda}{k}} \geq R/\sqrt{\lambda}$. Thus, the requirement on $\epsilon$ can be dropped if we assume $\estApprox{\lambda}{k} \in \DD{\lambda}{k-1}{\cc}$. Such condition is enforced in solver such as the one developed in \cite{NEURIPS2019_60495b4e}.

Finally, we put in application this result with next proposition, which gives a bound on the excess risk with inexact solver.
\begin{proposition}[Bound on the excess risk with inexact solver]\label{prop:bound_excess_risk_inexact_solver}
    Assume that:
    \begin{itemize}
        \item the requirement of \cref{prop:error_propagation_proximal_sequence} hold;
        \item the requirement of \cref{th:optimal_rates} hold, namely $\Hyp{1 - 5}$;
    \end{itemize}
    The first is an hypothesis on the \emph{optimization procedure}, while the second in an hypothesis on the \emph{statistics} of the learning task. Then, denoting $\estApprox{\lambda}{t}$ the approximation of $\estEmp{\lambda}{t}$ as defined in \cref{prop:error_propagation_proximal_sequence}, we have the following bound on the excess risk:
    \begin{equation*}
        \lossL{}{}{\estApprox{\lambda}{t}} - \lossL{}{}{\optimum} \leq \cstCBias \lambda^{2s}
        + \cstCVar \frac{\df_\lambda}{n} + \cstE{\cc} \, \epsilon, \quad s = (r+1/2) \land t,
    \end{equation*}
    with:
    \begin{equation*}
        \cstE{\cc} \eqdef 4 \Psi(4 - \log(1 - \cc)) \frac{e^4}{1-\cc} \kappa_\cc^2, \quad \text{e.g.} \quad \cstE{1/2} \leq 4.3 \cdot 10^3.
    \end{equation*}
\end{proposition}
\begin{proof}
    The proof boils down to combining the statistical results held in \cref{th:optimal_rates} with the optimization result of \cref{prop:error_propagation_proximal_sequence}. Begin by writing
    \begin{align*}
        \lossL{}{}{\estEmp{\lambda}{t}} - \lossL{}{}{\optimum} 
        &\leq \Psi\p{\tOp(\estApprox{\lambda}{t} - \optimum)} \norm{\estApprox{\lambda}{t} - \optimum}_{\Hess{}{}{\optimum}}^2 \\
        &\leq 2 \Psi\p{\tOp(\estEmp{\lambda}{t} - \estApprox{\lambda}{t}) + \tOp(\estEmp{\lambda}{t} - \optimum)} \csb{\norm{\estEmp{\lambda}{t} - \estApprox{\lambda}{t}}_{\Hess{}{}{\optimum}}^2 + \norm{\estEmp{\lambda}{t} - \optimum}_{\Hess{}{}{\optimum}}^2}.
    \end{align*}
    We know how to handle the statistical term $\norm{\estEmp{\lambda}{t} - \optimum}_{\Hess{}{}{\optimum}}^2$.
    
    ~\paragraph{Bound on $\tOp(\estEmp{\lambda}{t} - \estApprox{\lambda}{t})$.}
    As in the beginning of the proof of \cref{prop:constant_prefactor_variance}, we write:
    \begin{align*}
        \tOp(\estEmp{\lambda}{t} - \estApprox{\lambda}{t}) 
        &\leq \frac{1}{\rOp{\lambda}{\estApprox{\lambda}{t}}} \norm{\estApprox{\lambda}{t} - \estEmp{\lambda}{t}}_{\Hess{\lambda}{}{\estApprox{\lambda}{t}}} \\
        &\leq \frac{1}{\rOp{\lambda}{\estApprox{\lambda}{t}}} \phibot^{-1}\p{\tOp(\estApprox{\lambda}{t} - \estEmp{\lambda}{t})} \norm{\nabla \lossLEmp{\lambda}{t-1}{\estApprox{\lambda}{t}}}_{\Hess{\lambda}{-1}{\estApprox{\lambda}{t}}} \\
        &= \phibot^{-1}\p{\tOp(\estApprox{\lambda}{t} - \estEmp{\lambda}{t})} \NwtdecTrueNorm{\lambda}{t-1}{\estApprox{\lambda}{t}} \\
        &\leq \phibot^{-1}\p{\tOp(\estApprox{\lambda}{t} - \estEmp{\lambda}{t})} \cc
    \end{align*}
    where we used the fact that $\estApprox{\lambda}{t} \in \DD{\lambda}{t-1}{\cc}$, an assumption of \cref{prop:error_propagation_proximal_sequence}. With the same reasoning of \cref{eq:basic_localization_trick}, we conclude:
    \begin{equation*}
        \tOp(\estEmp{\lambda}{t} - \estApprox{\lambda}{t}) \leq - \log (1- \cc).
    \end{equation*}
    
    ~\paragraph{Bound on $\norm{\estEmp{\lambda}{t} - \estApprox{\lambda}{t}}_{\Hess{}{}{\optimum}}^2$.}
    Use a similar reasoning as we used for the variance. Under $\Hyp{1}$, we have (Proof of \cref{th:bound_on_bias}, 1st point)
    \begin{equation*}
        \norm{\estEmp{\lambda}{t} - \estApprox{\lambda}{t}}_{\Hess{}{}{\optimum}}^2 \leq 2 \norm{\estEmp{\lambda}{t} - \estApprox{\lambda}{t}}_{\HessEmp{\lambda}{}{\optimum}}^2.
    \end{equation*}
    First write:
    \begin{align*}
        \norm{\estEmp{\lambda}{t} - \estApprox{\lambda}{t}}_{\HessEmp{\lambda}{}{\optimum}}
        &\leq e^{\tOp(\estEmp{\lambda}{t} - \optimum)/2} e^{\tOp(\estApprox{\lambda}{t} - \estEmp{\lambda}{t})/2} \norm{\estEmp{\lambda}{t} - \estApprox{\lambda}{t}}_{\HessEmp{\lambda}{}{\estApprox{\lambda}{t}}},
    \end{align*}
    then, for each term, use:
    \begin{itemize}
        \item $\tOp(\estEmp{\lambda}{t} - \optimum) \leq 4$ (end of \cref{th:optimal_rates}) so that $e^{\tOp(\estEmp{\lambda}{t} - \optimum)/2} \leq e^2$;
        \item $\tOp(\estApprox{\lambda}{t} - \estEmp{\lambda}{t}) \leq -\log(1-\cc)$ so that $e^{\tOp(\estApprox{\lambda}{t} - \estEmp{\lambda}{t})/2} \leq (1 - \cc)^{-1/2}$;
        \item and finally:
        \begin{align*}
            \norm{\estEmp{\lambda}{t} - \estApprox{\lambda}{t}}_{\HessEmp{\lambda}{}{\estApprox{\lambda}{t}}}
            &\leq \phibot^{-1}\p{\tOp(\estApprox{\lambda}{t} - \estEmp{\lambda}{t})} \norm{\nabla \lossLEmp{\lambda}{t-1}{\estApprox{\lambda}{t}}}_{\Hess{\lambda}{-1}{\estApprox{\lambda}{t}}} \\
            &\leq \phibot^{-1}\p{\tOp(\estApprox{\lambda}{t} - \estEmp{\lambda}{t})} \NwtdecTrue{\lambda}{t-1}{\estApprox{\lambda}{t}} \\
            &\leq \csb{- \frac{1}{\cc} \log (1-\cc)} \epsilon \eqdef \kappa_\cc \epsilon.
        \end{align*}
    \end{itemize}
    
    ~\paragraph{Putting it all together.} 
    Thus, using the upper bound on the excess risk, we have with probability greater than $1-2\delta$
    \begin{equation*}
        \lossL{}{}{\estApprox{\lambda}{t}} - \lossL{}{}{\optimum} \leq \cstCBias \lambda^{2s}
        + \cstCVar \frac{\df_\lambda}{n} + 4 \Psi(4 - \log(1 - \cc)) \frac{e^4}{1-\cc} \kappa_\cc^2 \epsilon, \quad s = (r+1/2) \land t.
    \end{equation*}
    Taking $\cc = 1/2$, we have $\Psi(4 - \log(1 - \cc)) \leq 5$ and $\kappa_\cc \leq 1.4$, which allows bounding the quantity in front of $\epsilon$.
\end{proof}

\section{Technical lemmas}\label{sec:technical_lemmas_appendix}

\subsection{Concentration of Hermitian operators}

In this section, we import results from \cite{regularized-erm-gsc} and \cite{blanchard}. The former provides a bound on $\norm{\HessEmp{\lambda}{-1/2}{\theta} \Hess{\lambda}{1/2}{\theta}}$. The latter provides a bound on $\norm{\HessEmp{\lambda}{-1}{\theta} \Hess{\lambda}{}{\theta}}$, which is more difficult to obtain. They use the fact that $\df_\lambda = \Tr \Hess{\lambda}{}{\theta} \Hess{}{}{\theta}$ for least square, but we can't use this very convenient relation here. Thus, we only use their result in the case $1/2 < r <1$, which makes optimal rate still possible. 

We will only use
\begin{equation*}
    \Tr \Hess{\lambda}{-1}{\theta} \HessEmp{}{}{\theta} \leq \frac{\cstBTwo(\theta)}{\lambda}.
\end{equation*}

\begin{proposition}[Concentration bound]\label{prop:main_concentration_bound_operators}
    Let $\delta \in (0, 1]$ and $\lambda > 0$. The following holds:
    \begin{align}
        \label{eq:prop_concentration_sqrt}
        n \geq 24 \frac{\cstBTwo(\theta)}{\lambda} \log \frac{8 \cstBTwo(\theta)}{\lambda \delta} \quad &\implies \quad \norm{\HessEmp{\lambda}{-1/2}{\theta} \Hess{\lambda}{1/2}{\theta}} \leq \sqrt{2}, \\
        \label{eq:prop_concentration_inv}
        n \geq 8 \frac{\cstBTwo(\theta)^2}{\lambda^2} \log^2 \frac{2}{\delta} \quad &\implies \quad \norm{\HessEmp{\lambda}{-1}{\theta} \Hess{\lambda}{}{\theta}} \leq 2, \\
        \label{eq:prop_concentration_HS}
        n \geq 2 \p{1 \lor \frac{4 \cstBTwo(\theta)^2}{\lambda^{2s}}} \log \frac{2}{\delta} \quad &\implies \quad \norm{\Hess{}{}{\theta} - \HessEmp{}{}{\theta}}_{HS} \leq \lambda^s,
    \end{align}
    where each bound hold with confidence $1 - \delta$. 
\end{proposition}
\begin{proof}
    The first equation is Lemma 6 of \cite{regularized-erm-gsc}. The second equation can be adapted from Proposition 5.4 of \cite{blanchard}, except that we use
    \begin{equation*}
        \Tr \Hess{\lambda}{}{\theta} \Hess{}{}{\theta} \leq \frac{\cstBTwo(\theta)}{\lambda}
    \end{equation*}
    instead of $\df_\lambda$. For the last inequality, use Bernstein inequality for random vectors. With probability $1 - \delta$:
    \begin{equation*}
        \norm{\Hess{}{}{\theta} - \HessEmp{}{}{\theta}}_{HS} \leq \frac{2 \cstBTwo(\theta) \log 2/\delta}{n} + \cstBTwo(\theta) \sqrt{\frac{2 \log 2/\delta}{n}}.
    \end{equation*}
    Assuming $n \geq 2 \log 2/\delta$, this bound becomes
    \begin{equation*}
        \norm{\Hess{}{}{\theta} - \HessEmp{}{}{\theta}}_{HS} \leq 2 \cstBTwo(\theta) \sqrt{\frac{2 \log 2/\delta}{n}}.
    \end{equation*}
    Let $s > 0$. Further requiring $n \geq 8 \cstBTwo(\theta) \lambda^{-2s} \log 2/\delta$ gives: 
    \begin{equation*}
        \norm{\Hess{}{}{\theta} - \HessEmp{}{}{\theta}}_{HS} \leq \lambda^{s}
    \end{equation*}
    which completes the proof.
\end{proof}

\subsection{Inequalities on Hermitian operators}

The following results are given in \cite{blanchard}. We redo the proof to track down and upper bound the constants which are discarded in the original paper. 

\begin{lemma}[Hermitian operator inequalities]\label{lem:hermitian_inequalities}
    Let $A, B$ be two non-negative self-adjoint operators on $\HH$. Assume $\norm{A}, \norm{B} \leq \kappa$, where $\norm{\cdot}$ denotes the operator norm. Then:
    \begin{align}
        \label{eq:lem_hermitian_ineq_diff_rleq1}
        \forall r \leq 1, \quad \norm{A^r - B^r} &\leq \norm{A - B}^r \\
        \label{eq:lem_hermitian_ineq_diff_rg1}
        \forall r > 1, \quad \norm{A^r - B^r} &\leq w(r) \norm{A - B} \\
        \label{eq:lem_hermitian_ineq_prod}
        \forall r \leq 1, \quad \norm{A^r B^r} &\leq \norm{AB}^r
    \end{align}
    with $r 2^{\intpart{r}+1} \kappa^r$. 
\end{lemma}

\begin{proof}
    For the first point, refer to \cite{bhatia2013matrix} Theorem X.1.I, Eq. (X.2). For the third point, refer to Theorem IX.2.1 of the same book. It is also known as Cordes inequality \cite{Fujii1993NormIE}. The proofs involve positive semidefinite matrices but are directly applicable to non-negative self-adjoint Hermitian operators.

    For the second point, assume $\norm{A}, \norm{B} \leq 1$. Consider the function $f(x) = (1-x)^r$, defined for $\absv{x} \leq 1$. Its Taylor expansion reads:
    \begin{equation*}
        f(x) = \sum_{n \geq 0} a_n x^n, \quad a_n = \frac{(-1)^n}{n!}\prod_{k=1}^n (r - k + 1)
    \end{equation*}
    We have:
    \begin{equation*}
        \absv{\frac{a_{n+1} x^{n+1}}{a_n x^n}} = \absv{\frac{r - n}{n+1} \cdot x} \underset{n \to \infty}{\to} \absv{x}
    \end{equation*}
    so applying d'Alembert's rule, we have that the radius of the serie is $1$. Now, we have that:
    \begin{align*}
        A^r - B^r = f(\II - A) - f(\II - B) = \sum_{n\geq 0} a_n \csb{(\II - A)^n - (\II - B)^n} \\
        \implies \norm{A^r - B^r} \leq \sum_{n\geq 0} \absv{a_n} \norm{(\II - A)^n - (\II - B)^n}
    \end{align*}
    Using that $(\II - A)^n - (\II - B)^n = (\II - A) (\II - A)^{n-1} - (\II - B)^{n-1} - (B- A) (\II - B)^{n-1}$, we obtain:
    \begin{align*}
        \norm{(\II - A)^n - (\II - B)^n} 
        &\leq \norm{(\II - A) (\II - A)^{n-1} - (\II - B)^{n-1}} + \norm{(B- A) (\II - B)^{n-1}} \\
        &\leq \norm{(\II - A)^{n-1} - (\II - B)^{n-1}} + 1 \\
        &\leq n \norm{A - B}
    \end{align*}
    Denoting $g(x) = (1-x)^{r-1} = \sum b_n x^n$, we have $f'(x) = -r g(x)$ which gives $n \absv{a_n} = r \absv{b_n}$. Then:
    \begin{align*}
        \norm{A^r - B^r} 
        &\leq \norm{A - B} \sum_{n\geq 0} n \absv{a_n} \\
        &\leq r \norm{A - B} \sum_{n\geq 0} \absv{b_n}
    \end{align*}
    We can somewhat painfully upper bound this last term. Notice that for $n > r$, all the $b_n$ have the same sign $s = (-1)^{\intpart{r}}$. Thus, for $N>r$:
    \begin{align*}
        \sum_{n=0}^N \absv{b_n} 
        &= \sum_{n=0}^\intpart{r} \absv{b_n} + s \sum_{n=\intpart{r}}^N b_n \\
        &= \sum_{n=0}^\intpart{r} \absv{b_n} + s \lim_{x \to 1} \sum_{n=\intpart{r}}^N b_n x^n \\
        &\leq 2 \sum_{n=0}^\intpart{r} \absv{b_n} + \lim_{x \to 1} g(x) \\
        &\leq 2 \sum_{n=0}^\intpart{r} \frac{1}{n!}\prod_{k=1}^n (r - k + 1) \\
        &\leq 2 \sum_{n=0}^\intpart{r} \p{\begin{matrix}
            \intpart{r} \\ n
        \end{matrix}} = 2^{\intpart{r}+1}
    \end{align*}

    Finally, apply these properties to $A/\kappa, B/\kappa$ to obtain in general:
    \begin{equation*}
        \norm{A^r - B^r} \leq r 2^{\intpart{r}+1} \kappa^r \norm{A - B}
    \end{equation*}
\end{proof}

\subsection{Basic calculus}

This is a few line of computation, but useful in multiple places.
\begin{lemma}[Bound on residual of IT's spectral function]\label{lem:bound_residual_it_spectral_function}
    Let $r, t > 0$. Consider the following function defined on $\csb{0, \kappa}$: 
    \begin{equation*}
        h(\sigma) = \p{\frac{\lambda}{\lambda+\sigma}}^t \sigma^r.
    \end{equation*}
    Then:
    \begin{equation*}
        \sup_{0 \leq \sigma \leq \kappa} h(\sigma) \leq 
        \begin{cases}
            \p{r \cdot \frac{\lambda}{t}}^r \quad &\iftext \; r < t \\
            \p{\frac{\lambda}{\kappa + \lambda}}^t \kappa^r  \quad &\otwtext.
        \end{cases}
    \end{equation*}
\end{lemma}

\begin{proof}
    $h$ is differentiable and
    \begin{align*}
        h'(\sigma) &= \frac{\lambda^t \sigma^{r-1}}{\p{\sigma+\lambda}^{t+1}} \csb{\sigma(r-t) + r\lambda}.
    \end{align*}
    If $t \leq r$ , the regularization saturates and the maximum is in $\hat{\sigma} = \kappa$, which gives 
    \begin{equation*}
        \sup_\sigma h(\sigma) \leq \p{\frac{\lambda}{\kappa + \lambda}}^t \kappa^r \underset{\lambda \to 0}{\sim} \lambda^t \kappa^{r-t}.
    \end{equation*}
    Otherwise, if $t > r$, the maximum is in $\hat{\sigma} = \frac{r\lambda}{t - r}$ and it reads
    \begin{equation}\label{eq:basic_calculus_glue_1}
    \begin{aligned}
        \sup_\sigma h(\sigma) 
        &\leq \p{\frac{t-r}{t}}^t \p{\frac{r \lambda}{t-r}}^r \\
        &= \p{\frac{t-r}{t}}^{t-r} r^r \p{\frac{\lambda}{t}}^r.
    \end{aligned}
    \end{equation}
    We can rewrite the prefactor in front of $\p{\nicefrac{\lambda}{t}}^r$. First,
    \begin{equation*}
        \p{\frac{t-r}{t}}^{t-r} r^r 
        = \p{\frac{t-r}{t}}^t \p{\frac{rt}{t-r}}^r .
    \end{equation*}
    Then, use
    \begin{equation}\label{eq:basic_calculus_glue_2}
        \p{\frac{t-r}{t}}^t \leq e^{-r} \quad \text{when} \quad r < t.
    \end{equation}
    Also, 
    \begin{equation}\label{eq:basic_calculus_glue_3}
        \p{\frac{rt}{e(t-r)}}^r = \p{e\p{\frac{1}{r} - \frac{1}{t}}}^{-r} \leq (e/r)^{-r} \leq r^r.
    \end{equation}
    Use \cref{eq:basic_calculus_glue_2} and \cref{eq:basic_calculus_glue_3} on the upper bound of \cref{eq:basic_calculus_glue_1}, and the result is obtained.
\end{proof}

\section{Experiments}\label{sec:numerics_appendix}

\subsection{Technical details}

\paragraph{Splines.}
The spline kernel of order $q$ is defined on $\csb{0,1}^2$ as
\begin{equation*}
    \Lambda_q(x, z) = \sum_{k \in \ZZ} \frac{e^{2i\pi k(x-z)}}{\absv{k}^q}. \\
\end{equation*}
A closed form expression is available when $q$ is an even integer:
\begin{equation*}
    \Lambda_{q}(x, z) = 1 + \frac{(-1)^{q/2-1}}{q!} B_{q}(\absv{x - z}).
\end{equation*}
$B_{q}$ are Bernoulli polynomial of order $q$. They can be implemented easily. We also have the relation
\begin{equation*}
    \dotprod{\Lambda_q(x, \cdot)}{\Lambda_{q'}(x', \cdot)}_{\LLspace} = \Lambda_{q + q'}(x, x')
\end{equation*}

Our choice of $r, \alpha$ reflects the constraints on $\alpha$ and $(r+1/2)\alpha +1/2$ to be even integers. 

\paragraph{Regularization.}
For both least square and logistic regression, the regularization $\lambda$ is chosen among $50$ log spaced values between $10^{-4}$ and $1$. 

\paragraph{Resources.}
Computation was carried by a \texttt{Intel(R) Xeon(R) CPU E5-1620 v2 @ 3.70GHz}, with 32GB of RAM.

\subsection{Simulations with least square}\label{sec:numerics_ls_results}
\paragraph{Estimating $\estEmp{\lambda}{t}$.}
We leverage the very convenient filter interpretation with least square. We diagonalize the kernel matrix $K = U D U^\top$ once, then evaluate the estimator with
\begin{align*}
    \estEmp{\lambda}{t} &= \sum_{i=1}^n \alpha_i \phi(x_i), \\
    \alpha &= \frac{1}{n} U g_\lambda^t(D/n) D^\top y,
\end{align*}
where $g_\lambda^t$ is IT's filter, defined in \eqref{eq:it_spectral_function}. 

\paragraph{Simulations.}
The simulations are reported in \cref{fig:ls_results_sample_complexity,fig:ls_results_optimal_reg}. The same broad conclusion as for the classification task with the logistic loss apply. Surprisingly, $\IT{8}$ seems to suffer from higher constant than its counterpart with low $t$.

\begin{figure}[t]
    \centering
    \begin{subfigure}[t]{0.49\linewidth}
        \centering
        \includegraphics{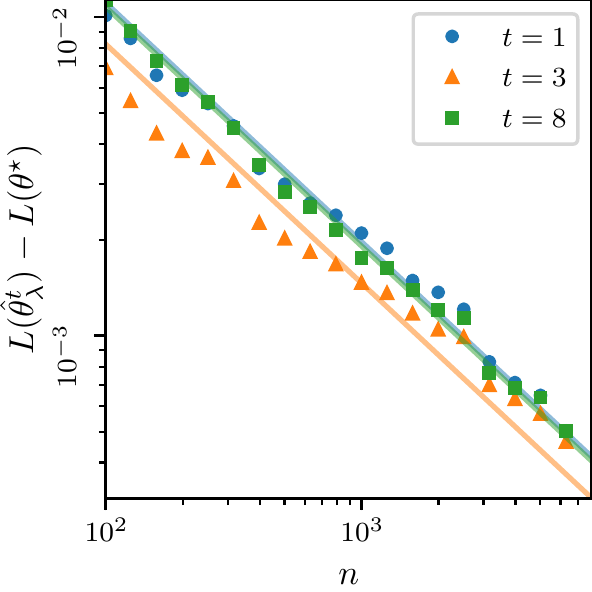}
    \end{subfigure}
    \hfill
    \begin{subfigure}[t]{0.49\linewidth}
        \centering
        \includegraphics{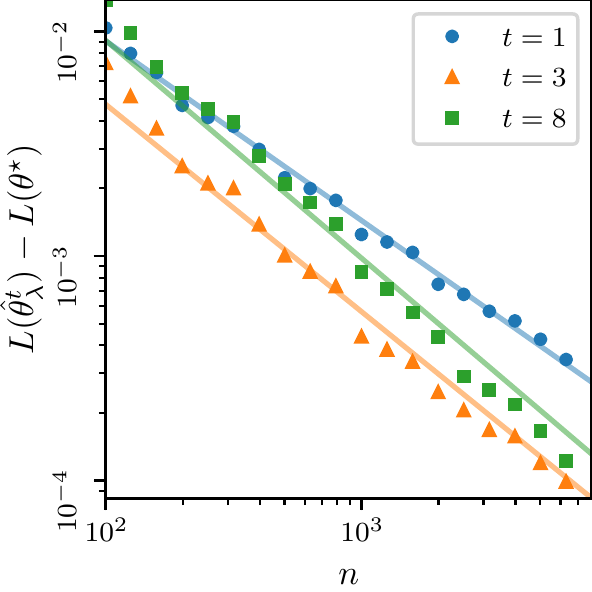}
    \end{subfigure}
    \caption{Excess risk with \textbf{least square} for various Iterated Tikhonov estimator, function of $n$. \textbf{Colors}: $t=1$ (Tikhonov) estimator is shown in orange; $t=2, 3$ in green, red. \textbf{Left}: from a difficult problem, $r=1/4, \alpha=2$. \textbf{Right}: easy problem, $r=41/4, \alpha=2$. Plain lines are predicted by theory, with slope $-\nicefrac{\alpha(1 + 2s)}{1 + \alpha(1 + 2s)}$, $s = \min\cb{r, t-1/2}$ (see main text). All plots are averaged over $100$ different initialization.}
    \label{fig:ls_results_sample_complexity}
\end{figure}

\begin{figure}[t]
    \begin{subfigure}[t]{0.49\linewidth}
        \centering
        \includegraphics{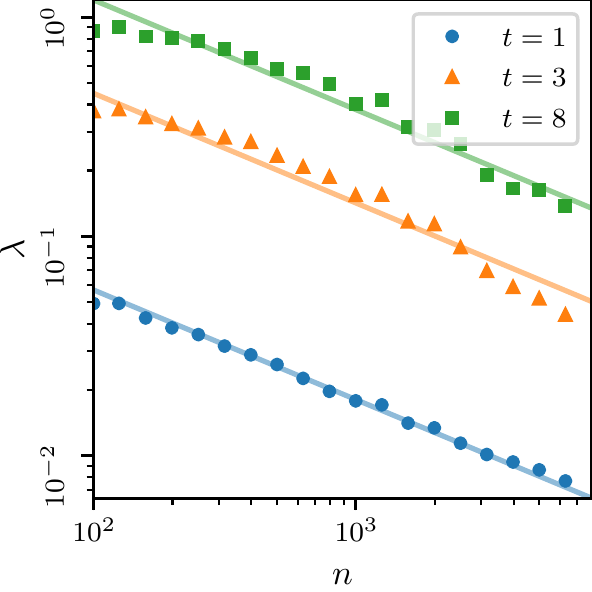}
    \end{subfigure}
    \hfill
    \begin{subfigure}[t]{0.49\linewidth}
        \centering
        \includegraphics{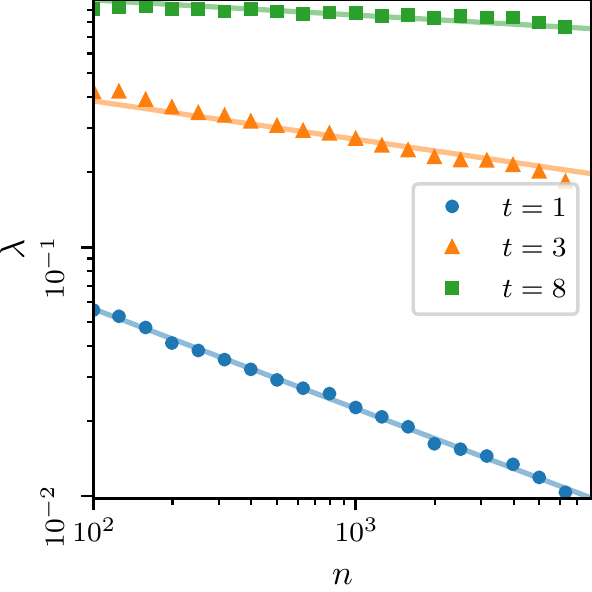}
    \end{subfigure}
    \caption{Chosen regularization $\lambda$ with \textbf{least square} for various Iterated Tikhonov estimator, function of $n$. \textbf{Colors}: $t=1$ (Tikhonov) estimator is shown in orange; $t=3, 8$ in green, red. \textbf{Left}: from a difficult problem, $r=1/4, \alpha=2$. \textbf{Right}: easy problem, $r=41/4, \alpha=2$. Plain lines are predicted by theory, with slope $-\nicefrac{\alpha}{1 + \alpha(1 + 2s)}$, $s = \min\cb{r, t-1/2}$ (see main text). All plots are averaged over $100$ different initialization.}
    \label{fig:ls_results_optimal_reg}
\end{figure}

\subsection{Synthetic binary task}\label{sec:numerics_binary}

\paragraph{Derivation of the noise.}
We have $\optimum(x) = \Lambda_{(r+1/2)\alpha+\epsilon}(x, 0)$ a function of smoothness $r+1/2$ in $\LLspace$. We want to use logistic regression. Thus, we need to choose the noise $\rho(y \mid x)$ so that
\begin{equation*}
    \optimum(x) = \arg \min_z \int_{\YY} \ell(y, z) \dd \rho_{y \mid x}(y).
\end{equation*}
To keep things simple, we restrict the output space to $\YY = \cb{-1, 1}$. Denote $a(x) = \PP(y=1 \mid x)$. We will have $\PP(y=-1 \mid x) = 1-a(x)$. Now we need to choose $a$ s.t
\begin{equation*}
    a \in \csb{0, 1} \quad \andtext \quad 
    \optimum(x) = \arg \min_z h(z) \eqdef \log(1 + e^z)(1-a) + \log(1 + e^{-z})a.
\end{equation*}
Having $a >0, 1-a>0$ implies that $h$ has a unique minimizer $z^*$. Then
\begin{equation*}
    h'(z) = \frac{1}{1 + e^z}\p{(1 - a)e^z - a} \implies a = \frac{e^{z^*}}{1 + e^{z^*}}.
\end{equation*}
Having required that $\optimum(x) = \arg \min_z h(z) \eqdef z^*$, we can use the following output distribution:
\begin{align*}
    \YY &= \cb{-1, 1} \\
    \PP(y = 1 \mid x) &= \frac{1}{1 + e^{-\optimum(x)}} \\
    \PP(y = -1 \mid x) &= \frac{1}{1 + e^{+\optimum(x)}} 
\end{align*}
which, in turn, ensures that $a(x) \in \csb{0, 1}$. 

\paragraph{Newton or first-order methods.}
In practice, the proximal operator is evaluated with a Newton method, or we use the toolbox Cyanure for big $n$ \cite{mairal2019cyanure}. Both are used with tolerance $10^{-10}$, that is machine precision for single precision.
Generally speaking, first-order methods are considered more performant than Newton methods. However, both practical and theoretical considerations motivate the use of second-order scheme in our statement of \cref{prop:error_propagation_proximal_sequence_main}. Firstly, preconditionated iterative solver such as the one used in \cite{NEURIPS2019_60495b4e} provide very efficient results for ill-conditioned problems. Secondly, the analysis of GSC loss functions is well-suited to second-order scheme, as the Newton decrement is a natural quantity to keep track of the optimization error. Measuring the error differently would require additional assumption on the loss function. 

\paragraph{Estimating the excess risk.}
The excess risk is estimated with Monte Carlo sampling, with $10^4$ points:
\begin{equation*}
    ER(\theta) - ER(\optimum) \approx \frac{1}{n_{MC}}\sum_{i=1}^{n_{MC}} \frac{1}{1 + e^{-\optimum(x_i)}} \log\p{\frac{1 + e^{-\theta(x_i)}}{1 + e^{-\optimum(x_i)}}} + \frac{1}{1 + e^{\optimum(x_i)}} \log \p{ \frac{1 + e^{\theta(x_i)}}{1 + e^{\optimum(x_i)}}}
\end{equation*}

\paragraph{Additional results.}
We report here the regularization $\lambda$ chosen function of $n$ and $t$ for various IT regularized estimators. We confirm that the penalty used for IT is larger than of Tikhonov, to compensate for the fitting induced by the additional proximal steps. We also compare the excess risk achieved by IT with the excess risk of Tikhonov, and observe consistent improvement for easy task with a sufficiently high number of samples. 

\begin{figure}[p]
    \centering
    \begin{subfigure}[t]{.49\linewidth}
        \centering
        \includegraphics{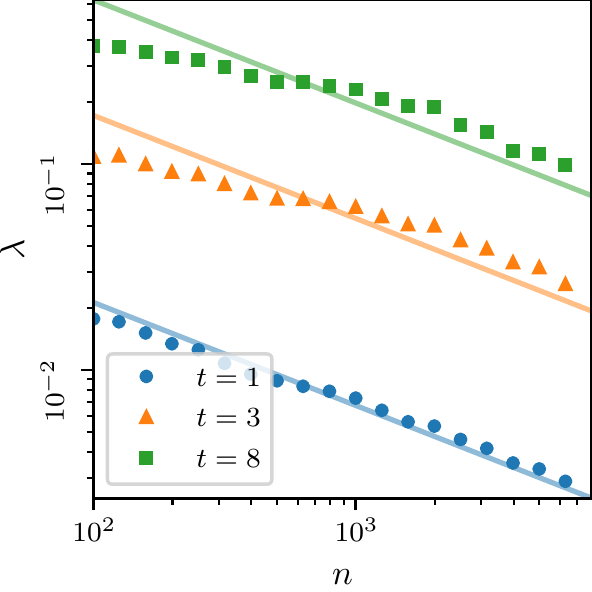}
    \end{subfigure}
    \hfill
    \begin{subfigure}[t]{.49\linewidth}
        \centering
        \includegraphics{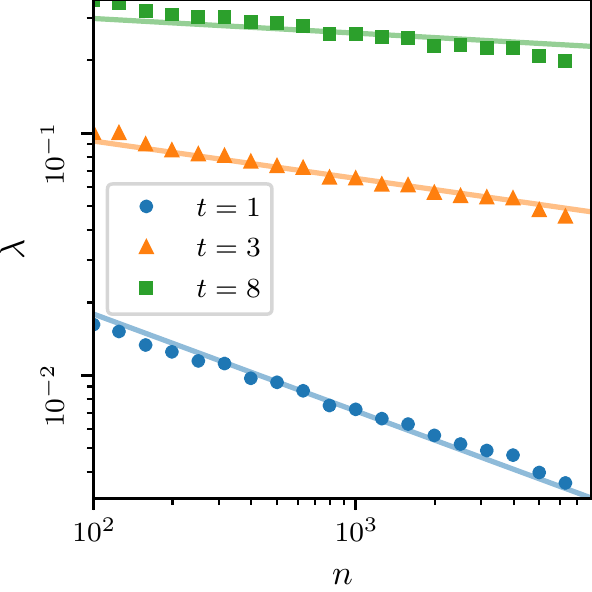}
    \end{subfigure}
    \caption{Chosen regularization $\lambda$ for various Iterated Tikhonov estimator, function of $n$. \textbf{Colors}: $t=1$ (Tikhonov) estimator is shown in orange; $t=3, 8$ in green, red. \textbf{Left}: from a difficult problem, $r=1/4, \alpha=2$. \textbf{Right}: easy problem, $r=41/4, \alpha=2$. Plain lines are predicted by theory, with slope $-\nicefrac{\alpha}{1 + \alpha(1 + 2s)}$, $s = \min\cb{r, t-1/2}$ (see main text). All plots are averaged over $100$ different initialization.}
    \label{fig:simulations_optimal_reg}
\end{figure}

\begin{figure}[p]
    \centering
    \begin{subfigure}[t]{.49\linewidth}
        \centering
        \includegraphics{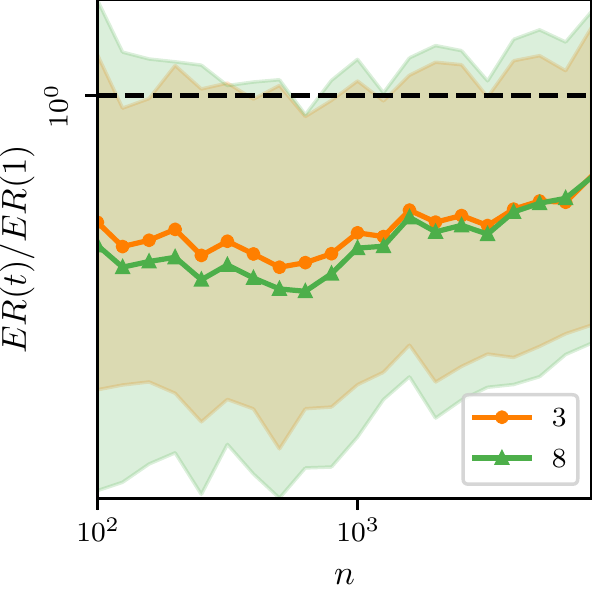}
    \end{subfigure}
    \hfill
    \begin{subfigure}[t]{.49\linewidth}
        \centering
        \includegraphics{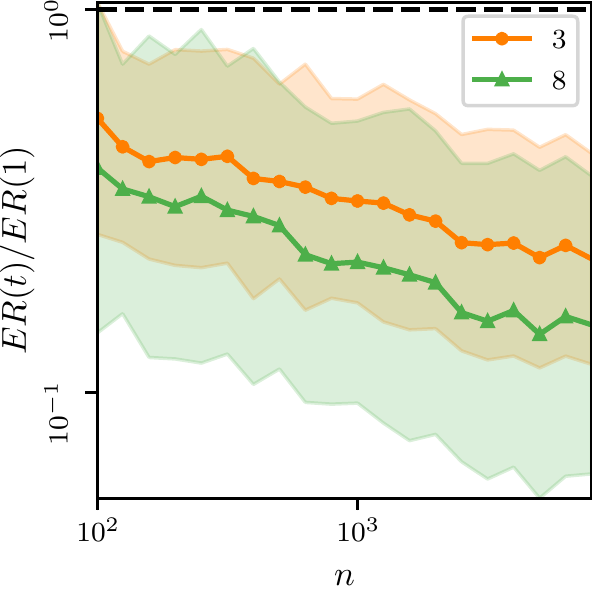}
    \end{subfigure}
    \caption{Ratio of IT's excess risk over Tikhonov's excess risk, function of $n$. \textbf{Left}: from a difficult problem, $r=1/4, \alpha=2$. \textbf{Right}: easy problem, $r=41/4, \alpha=2$. Whereas we expect the ratio to be consistently lower than 1, IT performs worse than Tikhonov in isolated cases, probably due to the optimization process and the chosen regularization path. Yet, it provides lower excess risk than Tikhonov overall, with up to an order of magnitude of improvement with as few as 1000 samples. All plots are averaged over $100$ different initialization.}
    \label{fig:it_beats_tikhonov}
\end{figure}

\end{document}